\PassOptionsToPackage{twoside=false}{geometry}
\documentclass[manuscript, screen]{acmart}
\acmJournal{TKDD}

\usepackage{microtype}
\usepackage{graphicx}
\usepackage{subfigure}
\usepackage{booktabs}
\usepackage{amsmath,amsfonts}

\usepackage{amssymb}
\usepackage{mathtools}
\usepackage{amsthm}
\usepackage{multirow}
\usepackage{tabularx}
\theoremstyle{plain}
\usepackage{bm}

\newtheorem{theorem}{Theorem}[section]
\newtheorem{proposition}[theorem]{Proposition}

\theoremstyle{definition}

\newtheorem{assumption}[theorem]{Assumption}
\theoremstyle{remark}
\newtheorem{remark}[theorem]{Remark}
\AtBeginDocument{%
  }

\begin{document}
\renewcommand\footnotetextcopyrightpermission[1]{} 
\settopmatter{printacmref=false} 
\title{CGRL: Causal-Guided Representation Learning for Node-Level Out-of-Distribution Generalization}


\author{Bowen Lu}
\affiliation{%
  \institution{Anhui University}
  \city{Hefei}
  \country{China}}
\email{wa24201032@stu.ahu.edu.cn}

\author{Lianqiang Yang}
\authornotemark[1]
\affiliation{%
  \institution{Anhui University}
  \city{Hefei}
  \country{China}}
\email{yanglq@ahu.edu.cn}
\author{Teng Li}
\affiliation{%
 \institution{Anhui University}
 \city{Hefei}
 \country{China}}
\email{tenglwy@gmail.com}
\author{Kun Zhang}
\affiliation{%
  \institution{Carnegie Mellon University}
  \city{Pittsburgh}
  \state{Commonwealth of Pennsylvania}
  \country{USA}}
\email{kunz1@cmu.edu}

\renewcommand{\shortauthors}{Trovato et al.}

\begin{abstract}
Graph Neural Networks (GNNs) have achieved remarkable success in graph-related tasks, yet their performance often degrades substantially in out-of-distribution (OOD) settings. Under distribution shifts, GNNs tend to fit environmental noise and spurious correlations rather than stable causal mechanisms, resulting in poor OOD robustness and unstable predictive representations. Although recent methods attempt to address this issue through environment invariance or structural causal reasoning, they remain insufficient for node classification because they do not explicitly capture the fine-grained latent geometry required by the task. We further identify a training instability phenomenon, termed \emph{Info-Jitter}, in which the mutual information between predictive representations and ground-truth labels fluctuates during training. To address these problems, we formulate a node-classification-specific causal graph from the geometric objective of node classification. Based on this formulation, we use $do$-calculus to block non-causal paths induced by environmental noise, derive a deconfounded interventional objective, and further develop a variational lower bound to disentangle the predictive representations into intra-class and inter-class component. Specifically, we propose a novel Causal-Guided Representation Learning (CGRL) framework. It consists of two components: a causal representation learning module based on multi-branch re-weighted representation learning (RRL), which learns a causal modulation matrix to emphasize causally relevant signals and suppress interference of environmental noise during message passing; and an optimization strategy that integrates intra-class aggregation, inter-class separation, energy-based reconstruction, and supervised prediction to regularize the latent space toward robust node-level generalization. Experiments on diverse benchmark datasets show that CGRL outperforms strong baselines across multiple types of distribution shifts and effectively alleviates the \emph{Info-Jitter} phenomenon.
\end{abstract}




\maketitle

\section{Introduction}

Graph-structured data \cite{NEURIPS2020_fb60d411} are ubiquitous in real-world applications \cite{Abbahaddou2024GraphNN, qiao2025gcal}, including social networks \cite{kuhne2025optimizing}, recommendation systems \cite{10.1145/3811912}, and biological interaction networks \cite{10.1145/3711896.3737144}. To model such non-Euclidean data, Graph Neural Networks (GNNs) have become a dominant paradigm, as they effectively integrate graph topology with node attributes for a wide range of downstream tasks, such as node classification \cite{10.1145/3543507.3583335, jeong2024igraphmix}, link prediction \cite{kuhne2025optimizing}, and graph classification \cite{10.1145/3638781}. Despite their strong in-distribution (ID) performance \cite{NEURIPS2024_91813e5d, luo2025node}, however, GNNs often suffer substantial degradation under out-of-distribution (OOD) settings \cite{wu2022discovering}, where the training and test graphs differ in feature distributions, structural patterns, or latent environments. This challenge is particularly severe for node classification, because each node is embedded in a densely interconnected graph and its representation is continuously shaped by neighborhood aggregation, graph topology, and environment-dependent context. Therefore, graph distribution shifts do not merely perturb the input data, but can also alter the representation formation process itself. Under such shifts, GNNs tend to fit environmental noise and spurious correlations rather than stable causal mechanisms, which degrades OOD robustness and destabilizes predictive representations.

A natural line of research addresses this challenge through causal representation learning \cite{10.1145/3644392, COHF}. Existing graph OOD methods can be broadly grouped into two categories. The first seeks to learn environment-invariant representations under the assumption that the underlying causal mechanisms remain stable across environments, whereas spurious correlations vary with environmental changes \cite{10.1145/3644392, wu2022handling}. By enforcing representation consistency across multiple environments, these methods aim to identify stable predictive factors and improve OOD generalization \cite{NEURIPS2022_9b77f073, chen2023does}. The second category approaches graph OOD generalization through Structural Causal Models (SCMs) \cite{primer}, which explicitly model causal relations among graph structure, latent variables, and labels, and then use intervention tools such as backdoor adjustment to mitigate confounding bias \cite{NEURIPS2022_8b21a7ea, CIA-LRA}. Although both paradigms have shown promise, they remain insufficient for node classification. Environment-invariant methods typically require explicit environment partitions, which are often unavailable in practical node-level scenarios and may themselves introduce artificial shifts \cite{creager21environment, FLOOD}. More importantly, imposing a global invariance constraint may suppress not only harmful environmental noise but also informative local topology and class-relevant neighborhood structure that are essential for distinguishing nearby nodes \cite{NEURIPS2022_9b77f073}. SCM-based methods, in contrast, provide a principled deconfounding framework, yet most of them operate at a relatively macroscopic level: they focus on blocking confounding paths without explicitly characterizing the latent geometric objective intrinsic to node classification, namely that intra-class nodes should become compact while inter-class nodes should remain separated in the representation space \cite{10.1145/3707645}. Under distribution shifts, however, intra-class nodes is often contaminated by heterophilic neighbors and environment-dependent perturbations. As a result, predictive representations contain both stable intra-class and interfering inter-class component. Without explicit latent disentanglement, existing causal methods cannot separately model these two parts, leaving predictive representations vulnerable to OOD scenarios.

\begin{figure}[t]
	\begin{center}
		\centerline{\includegraphics[width=1\columnwidth]{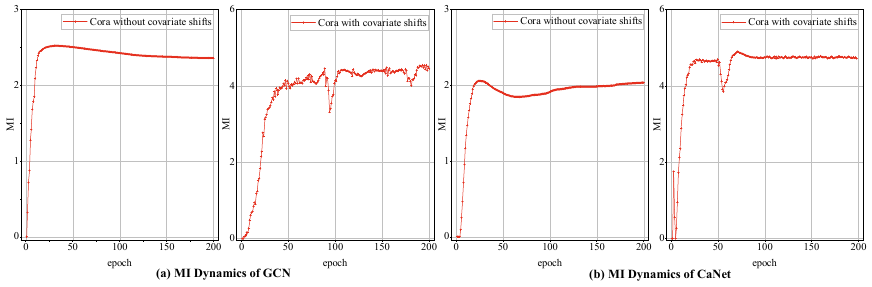}}
		\caption{Illustration of the \emph{Info-Jitter} phenomenon under covariate shifts. We track the mutual information (MI) between predictive representations and ground-truth labels across training epochs on the Cora dataset. Covariate shifts are induced by degree-based partitioning, where high-degree nodes are used for training and low-degree nodes for OOD data. Under identical distributions (left panel of (a) and left panel of (b)), both the conventional GCN and the causal baseline CaNet exhibit smooth and stable MI convergence after an initial rise. In contrast, under covariate shifts (right panel of (a) and right panel of (b)), both models display pronounced MI oscillations throughout training, indicating unstable representation-label alignment and revealing their vulnerability to environment-dependent spurious correlations.}
		\label{figure1}
		\vspace{-1cm}
	\end{center}
\end{figure}

To make this limitation explicit, we identify a phenomenon that we term \emph{Info-Jitter}: the mutual information (MI) between predictive representations and ground-truth labels becomes unstable during training in OOD settings. Figure \ref{figure1} provides an illustrative example on the Cora dataset \cite{synthetic}, where covariate shifts are induced by degree-based partitioning. In the absence of distribution shifts, both a Graph Convolutional Network (GCN) \cite{kipf2017semisupervised} and the causal baseline CaNet \cite{CaNet} exhibit smooth MI evolution and eventually reach stable convergence. Once the distribution shifts, however, this stability is lost and the MI trajectories become markedly oscillatory. Although Figure \ref{figure1} only presents one representative case, this phenomenon is not restricted to a specific backbone or a single shift type \cite{heyunfei}. This instability is not merely a numerical artifact; rather, it reveals that the learned representation--label relationship itself becomes fragile. Because these methods do not explicitly disentangle intra-class representation from inter-class  interference, the learned predictive representations may still absorb environment-dependent patterns that are predictive on the training graph but unreliable on unseen data. When such non-causal correlations are disrupted by distribution shifts, the latent space itself becomes unstable, and the resulting decision boundary deteriorates accordingly. Therefore, robust node-level OOD generalization requires more than coarse-grained deconfounding from environmental noise. It also calls for a finer-grained disentanglement formulation, so that representation learning itself remains stable under distribution shifts \cite{xu2024causal, ICLR2024_6e2a1a8a, Zhang2024DebiasingGR}.

Motivated by this observation, we revisit node classification from a geometric and causal perspective. At its core, node classification is not merely a prediction problem, but also a structural one. This suggests that robust node-level prediction requires the learning of stable intra-class structure while suppressing inter-class interference induced by environmental confounding. From this viewpoint, distribution shifts do not only perturb the input graph, but also distort the latent geometry by introducing spurious overlap between classes. Based on this formulation, we construct a causal graph tailored to node classification and use \emph{do}-calculus to identify a deconfounded interventional objective. We further derive a variational lower bound that provides a principled objective for node-level causal disentanglement by separating intra-class causal structure from inter-class interference. To instantiate these theoretical insights, we propose a novel \textbf{Causal-Guided Representation Learning (CGRL)} framework. It consists of two tightly coupled components. The first is a causal representation learning module built upon multi-branch re-weighted representation learning (RRL), which learns a causal modulation matrix to adaptively emphasize causally relevant signals in the predictive representation and suppress environment-sensitive interference during message passing. The second is an optimization strategy consisting of four complementary objectives: intra-class aggregation, inter-class separation, energy-based reconstruction, and supervised prediction. Intra-class aggregation promotes compact class-consistent structure, while inter-class separation suppresses spurious overlap across different classes. The energy-based reconstruction objective further regularizes the latent graph structure in a deconfounded manner, and the supervised prediction objective ensures that the learned representation remains aligned with the node classification task. Through this joint optimization, CGRL not only filters out environmental interference, but also reshapes the latent space toward the fundamental geometric requirement of node classification. Extensive experiments on multiple benchmark datasets demonstrate that CGRL consistently outperforms state-of-the-art methods under graph OOD settings and substantially alleviates the \emph{Info-Jitter} phenomenon.

Our main contributions are summarized as follows:
\begin{itemize}
	\item We identify a training instability phenomenon under distribution shifts, termed \emph{Info-Jitter}, and relate it to the entanglement between intra-class causal structure and inter-class spurious interference.
	
	\item We formulate a node-classification-specific causal graph and derive a variational lower bound that provides a principled basis for disentangling causal representation.
	
	\item We propose \textbf{CGRL}, a causal-guided representation learning framework that integrates multi-branch causal modulation, geometric surrogate optimization, energy-based reconstruction, and supervised prediction for robust node-level OOD generalization.
	
	\item We conduct extensive experiments on diverse graph benchmarks, showing that CGRL consistently improves OOD generalization and alleviates the \emph{Info-Jitter} phenomenon.
\end{itemize}

\section{RELATED WORK}

\subsection{Graph Neural Networks}
Graph Neural Networks (GNNs) are primarily designed to handle graph-structured data with non-Euclidean structures \cite{graphsage, DBLP:conf/iclr/XuHLJ19, lee2023guiding}. In the task of node classification, they update the representation of each node by aggregating topological information and feature attributes among its neighbors, attempting to cluster intra-class nodes together and separate inter-class nodes, so as to achieve effective node discrimination. Typical GNNs include Graph Convolutional Network (GCN) \cite{kipf2017semisupervised} and Graph Attention Network (GAT) \cite{veličković2018graph}. The former aggregates node information in the spectral domain, while the latter incorporates attention mechanisms over nodes in the spatial domain. However, neither of them can stably learn the mutual information between predictive representations and ground-truth labels.

\subsection{Distribution Shifts and Causal Representation Learning}

Out-of-distribution (OOD) generalization on graphs is fundamentally challenged by the discrepancy between training and test data \cite{NEURIPS2021_1c336b80}. Existing studies characterize such discrepancies from multiple perspectives \cite{gui2022good, wu2022handling}. From the viewpoint of graph components, shifts may arise in node features, graph structures, environments, labels, or other latent factors \cite{DBLP:conf/icml/YehudaiFMCM21, CaNet, NEURIPS2022_9b77f073}. Such taxonomies help reveal the concrete sources of OOD difficulty in node classification and motivate methods tailored to different graph scenarios. For example, Explore-to-Extrapolate Risk Minimization (EERM) \cite{wu2022handling} simulates multiple virtual environments from a single observed graph, thereby modeling environment-related variation and improving extrapolation under graph shifts. From the perspective of the joint distribution over inputs and outputs, distribution shifts are also commonly categorized into covariate shifts and concept shifts \cite{gui2022good, NEURIPS2021_1c336b80, CIA-LRA}. The Graph Out-of-distribution (GOOD) benchmark \cite{gui2022good}, for instance, provides a controlled evaluation protocol that explicitly constructs both covariate shifts and concept shifts settings on graphs, and has become a standard testbed for studying node-level OOD generalization. Although these categorizations emphasize different aspects of distribution shift, they consistently suggest that GNNs are prone to fitting unstable, environment-dependent spurious correlations from in-distribution (ID) data, which in turn degrade OOD performance and destabilize node representation learning \cite{NEURIPS2022_8b21a7ea, COHF}.

To address this issue, an important line of work studies causal graph representation learning, which integrates causal inference with graph representation learning to distinguish stable causal mechanisms from spurious correlations \cite{NEURIPS2023_8ed2293e, wu2024learning}. Existing approaches can be broadly grouped into two directions. The first relies on causal invariant learning, where multiple environments are leveraged to identify stable predictive factors. While effective in some settings, such methods typically require explicit environment annotations or environment partitioning \cite{chen2023does, creager21environment, NEURIPS2022_4d4e0ab9}, which are often unavailable or difficult to define in practical node-level tasks. The second direction approaches node-level graph OOD generalization through structural causal reasoning, using causal graphs and intervention principles to reduce confounding effects without relying on explicit environment labels. For instance, Causal Intervention for Network Data (CaNet) \cite{CaNet} attributes node-level OOD degradation to latent confounding bias and introduces a pseudo-environment estimator to approximate latent environments, thereby enabling causal intervention without ground-truth environment annotations. Despite encouraging progress, however, current methods still focus primarily on coarse-grained invariance or macroscopic deconfounding. As a result, they leave limited room for the fine-grained latent disentanglement required by node classification, where robust generalization depends not only on reducing environmental confounding, but also on explicitly preserving intra-class structure while suppressing inter-class interference.

\subsection{Causal Invariance for OOD Generalization}

Causal invariance provides a widely used principle for OOD generalization: if a model can identify the representation component governed by stable causal mechanisms, its prediction should remain robust when environment-dependent correlations change \cite{gao2024rethinking, DBLP:conf/eccv/AnZJMH24}. In graph learning, this principle has been explored mainly in graph-level settings, where the prediction target is associated with a relatively compact and semantically coherent substructure. For instance, Causality Inspired Invariant Graph LeArning (CIGA) \cite{chen2023does} characterizes distribution shifts using structural causal models and argues that graph OOD generalization can be achieved when the predictor focuses on invariant subgraphs that preserve the causal information of labels. Based on this view, it decomposes graph prediction into invariant subgraph identification and classification, and uses an information-theoretic objective to maximize invariant intra-class information. This formulation is natural for graph classification, where the causal signal can often be localized to a relatively self-contained subgraph \cite{yu2025causal}. Although some existing methods implicitly or explicitly treat the ego-graph as an invariant subgraph for node-level prediction \cite{10.1145/3604427}, such a formulation is often too coarse for realistic node classification, which is performed on a single input graph where nodes are tightly coupled through neighborhood interactions and message passing \cite{jeong2024igraphmix, luo2025node}. Therefore, node classification imposes a finer-grained geometric requirement on the latent space: robust prediction depends not only on preserving stable causal information, but also on maintaining intra-class compactness and inter-class separability. This requirement is not directly captured by graph-level invariant subgraph formulations, nor by coarse-grained environment partitioning or macroscopic SCM-based deconfounding alone.

Accordingly, although existing causal-invariance methods can mitigate confounding bias at a high level, they do not directly address the fine-grained latent structure required by node-level OOD prediction. Under distribution shifts, intra-class nodes may still be contaminated by heterophilic neighbors and environmental perturbations, causing intra-class structure and inter-class interference to become entangled in the learned representation. This limitation motivates the need for a finer-grained causal formulation for node classification, in which OOD robustness is achieved not only through deconfounding, but also through explicit latent disentanglement that promotes intra-class aggregation and suppresses spurious inter-class overlap.

\begin{table*}[htbp]
	\centering
	\caption{Notations and corresponding meanings}
	\label{tab:gnn_symbols}
	\begin{tabular}{cccc}
		\toprule
		\textbf{Notation} & \textbf{Meaning} & \textbf{Notation} & \textbf{Meaning} \\
		\midrule
		$\mathcal{G}$ & Input graph & $G$ & Causal graph \\
		$\mathcal{V}$ & Node set of graph $\mathcal{G}$ & $\phi$ & Encoder parameters \\
		$\mathcal{E}$ & Edge set of graph $\mathcal{G}$ & $v$ & Node of $\mathcal{V}$ \\
		$E$ & Environment of $\mathcal{G}$& $u$ & Neighbor nodes of $v$ \\
		$\mathbf{Y}$ & Ground-truth label& $\mathcal{G}_v$ & Ego-graph of $v$ \\
		$N$ & Number of nodes & $\theta$ & Classifier parameters \\
		$|\mathcal{E}|$ & Number of edges & $\mathbf{Z}$ & Representation of GNNs \\
		$\mathbf{A}$ & Adjacency matrix of graph $\mathcal{G}$ & $\mathbf{Z}_p$ & Output causal representation \\
		$\mathbf{Z}_a$ & Intra-class causal representation & c & Class of node $v$\\
		$\mathbf{Z}_r$ & Inter-class spurious representation & $h_\theta(\cdot)$ & Classifier with parameter $\theta$ \\
		$f$ & Encoder & $\mathcal{R}_{\mathcal{T}}(f)$ & Expected target risk\\
		$\hat{\mathcal{R}}_{\mathcal{S}}(f)$ & Empirical source risk & $d_{\mathcal{H}\Delta\mathcal{H}}(\cdot, \cdot)$ & Class-conditional causal discrepancy\\
		$\lambda$ & Ideal joint risk & $\mathcal{Y}$ & Set of all classes \\
		$K$ & Branches of model & $deg(v)$ & Degree of node $v$ \\
		$\Gamma$ & Causal modulation matrix & $\mathcal{D}$ & Euclidean distance \\
		$\mathbf{z}_{p, v}$ & Representation of node $v$ in $\mathbf{z}_p$ & $\mathbf{z}_{p, u}$ & Representation of $u$ in $\mathbf{z}_p$\\
		$\mathbf{z}_{a,i}$ & Representation of class $c$ in $\mathbf{Z}_a$ & $\mathbf{z}_{r,j}$ & Representatio of class $c$ in $\mathbf{Z}_r$ \\
		\bottomrule
	\end{tabular}
\end{table*}
\section{PRELIMINARIES}
\subsection{Notations}
Let $\mathcal{G}=(\mathcal{V}, \mathcal{E})$ denotes an input graph with node set $\mathcal{V}$ and edge set $\mathcal{E}$. The set $\mathcal{G}=\{\mathcal{G}_1,\mathcal{G}_2,...,\mathcal{G}_{v},...\}$ consists of multiple ego-graphs, where $\mathcal{G}_v$ denotes the graph belonging to the $v$-th ego-graph. Let $N=|\mathcal{V}|$ represent the number of nodes. $A\in \mathbb{R}^{N\times N}$ and the  $\mathbf{X} \in \mathbb{R}^{N\times d}$ denote adjacency matrix and node attributes, respectively, where $d$ is representation dimension. For a node $v \in \mathcal{V}_c$, $N_{(v)}=\{u|\langle v,u\rangle\in \mathcal{E}\}$ is defined as the set of its neighbors and $deg(v)$ is its degree, and c is class of node $v$. For an edge $\langle u,v \rangle\in \mathcal{E}$, we have $A[u,v]=1$; otherwise, $A[u,v]=0$. We use $\mathbf{Z}_p$, $\mathbf{Y}$, $\mathcal{Y}$ to represent the predictive representation, ground-truth label and set of all classes, respectively. The goal of node classification is to cluster intra-class nodes and separate inter-class nodes. Thus, we denote the intra-class representation and inter-class representation by $\mathbf{Z}_a$ and $\mathbf{Z}_r$, respectively. The notations and meanings used in this paper are shown in Table \ref{tab:gnn_symbols}.

\subsection{Problem Definition}
In this paper, we focus on node classification. Unlike graph-level prediction, node classification is governed by a finer-grained latent-space objective: the learned representation should preserve class-consistent structure while avoiding cross-class interference. This distinction makes the standard notion of graph-level causal invariance difficult to transfer directly. In graph classification, causal invariance is often instantiated as an environment-invariant subgraph, and many existing methods are designed to identify such graph-level stable components \cite{10.1145/3767162, 10.1145/3690624.3709203, 10.1145/3534678.3539366}. In node-level settings, however, nodes are connected through intricate graph dependencies and continuously exchange information through message passing. As a result, the fine-grained invariant component associated with an individual node is substantially harder to isolate. Motivated by this mismatch, we construct the causal graph shown in Figure \ref{image2a}. A key aspect of our formulation is that the predictive representation $\mathbf{Z}_p$ should not be viewed as a single homogeneous latent variable. Instead, from a node-level perspective, it contains two structurally different components: an intra-class representation $\mathbf{Z}_a$, which captures the stable class-consistent structure to be preserved across environments, and an inter-class representation $\mathbf{Z}_r$, which reflects cross-class interference induced by environment-dependent spurious correlations. Figure \ref{image2b} provides an intuitive illustration of these two components. In practice, message passing in GNNs often mixes useful causal signals with heterogeneous neighborhood information, so that intra-class structure can be contaminated by representations from different-class neighbors, creating spurious overlap in the latent space. Such overlap is typically induced by fitting environment-dependent correlations rather than genuine class-relevant mechanisms, and it directly undermines OOD generalization.

Accordingly, the central problem studied in this paper is how to suppress environmentally induced spurious correlations, recover stable intra-class structure, and disentangle $\mathbf{Z}_a$ and $\mathbf{Z}_r$ from the predictive representation $\mathbf{Z}_p$. To address this problem, we combine causal deconfounding with a geometry-aware formulation of node classification, so that the learned representation is aligned with both the causal mechanism and the intrinsic latent-space objective of the task under distribution shifts.

\begin{figure}[t]
	\centering
	\subfigure[Causal graph $G$]{
		\includegraphics[width=0.6\textwidth]{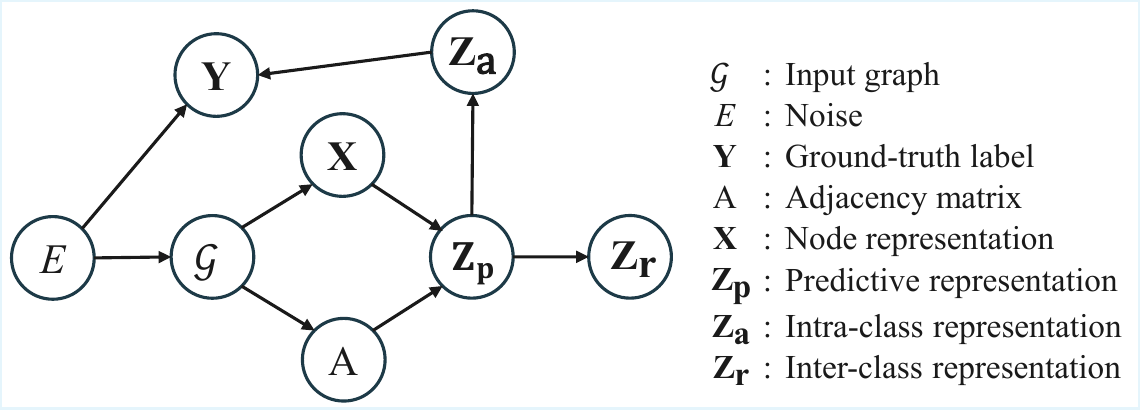}
		\label{image2a}
	}
	\subfigure[$\mathbf{Z}_a$ and $\mathbf{Z}_r$]{
		\includegraphics[width=0.3\textwidth]{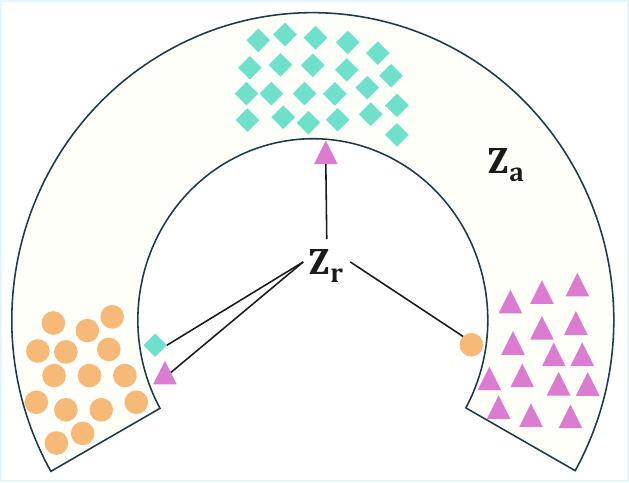}
		\label{image2b}
	}
	\caption{(a) A causal graph G consisting of 8 variables. (b) $\mathbf{Z}_a$ within the ring and $\mathbf{Z}_r$ outside the ring respectively represent intra-class and inter-class nodes, where nodes of different colors and types represent different categories.}
	\label{image2}
\end{figure}

\subsection{Energy-based Model}
Energy-based Model (EBM) \cite{ebm, DBLP:conf/iclr/GrathwohlWJD0S20, wu2023energybased} capture dependencies among variables by assigning scalar energies to them. In the OOD setting, the distribution discrepancy across graphs can be substantial. This calls for a robust method that enables the model to learn causal features. Instead of leveraging Variational Autoencoders (VAEs) \cite{DBLP:journals/corr/KingmaW13} and diffusion models \cite{NEURIPS2020_4c5bcfec}, we use EBM which exhibit strong compatibility with graph data and are capable of capturing causal structural dependencies among nodes. It define an energy function $E(X_1,X_2)$ between two variables $X_1$ and $X_2$. Then, its probability distribution is represented as follows:
\begin{align}
	P(X_1,X_2)=\frac{e^{- E(X_1,X_2)}}{\sum_x e^{- E(X_1,X_2)}}.
\end{align}
The energy score $E(X_1,X_2)$ can model the dependencies between $X_1$ and $X_2$, which can  serve as a key method for learning stable mutual information between prediction representations and ground-truth labels. However, current research on energy-based modeling and its applications in causal representation learning remain underexplored, and the potential of energy-based models in the stable modeling of interdependent data has not yet been fully exploited.

\section{METHODOLOGY}

In this section, we first analyze the reason why GNNs fail to stably learn the mutual information between predictive representations and ground-truth labels and GNNs struggle to generalize to unseen data in Section \ref{section4.1}. Then, Based on the this analysis, we derive a theoretical lower bound to guide model optimization. Finally, leveraging these theoretical foundations, we introduce a novel Causal-Guided Representation Learning (CGRL) framework tailored to address OOD problems on graphs.
\begin{figure*}[t]   
	
	\centering
	
	\includegraphics[width=1.0\textwidth]{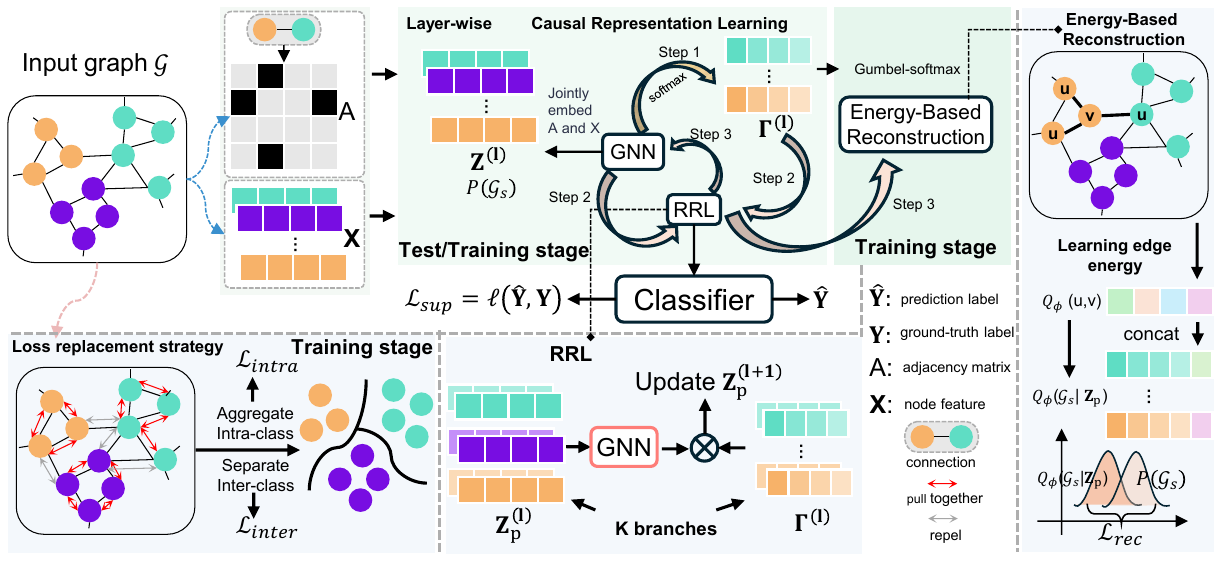}
	
	
	\caption{CGRL consists of two components: \emph{causal representation learning} and \emph{optimization strategy}. The former is implemented by a multi-branch re-weighted representation learning (RRL) module, which learns a causal modulation matrix $\Gamma$ to adaptively reweight the predictive representation $\mathbf{Z}_p$. The latter contains four objectives: intra-class aggregation, inter-class separation, energy-based reconstruction, and prediction loss. Given the adjacency matrix $A$ and node feature $\mathbf{X}$, a GNN encoder first produces the latent representation $\mathbf{Z}$, which parameterizes the graph prior $P(\mathcal{G}_v)$. During training, $\mathbf{Z}$ is further processed by the Gumbel-softmax, the RRL module, and the energy-based reconstruction module to obtain $\mathbf{Z}_p$, which is then fed into a classifier for prediction and supervision via $\mathcal{L}_{sup}$. In parallel, two geometric surrogate objectives are imposed on the latent space, yielding the intra-class loss $\mathcal{L}_{intra}$ and inter-class loss $\mathcal{L}_{inter}$. Together with the reconstruction loss $\mathcal{L}_{rec}$, these components jointly encourage deconfounded representation learning and robust OOD generalization.}
	
	\label{Figure2}
	
\end{figure*}

\subsection{A Causal Perspective for OOD Generalization}
\label{section4.1}
In this section, we analyze the \emph{Info-Jitter} phenomenon and the poor generalization of GNNs from a causal perspective \cite{causality}. As illustrated in Figure \ref{image2a}, we formulate a causal graph based on the SCM \cite{primer}, where directed arrows between two variables denote causal relationships. This graph comprises eight elements: environmental noise $E$, input graph $\mathcal{G}$, adjacency matrix $A$, node feature $\mathbf{X}$, predictive representation $\mathbf{Z}_p$, intra-class representation $\mathbf{Z}_a$, and inter-class representation $\mathbf{Z}_r$. The underlying causal mechanisms are detailed as follows:

\begin{itemize}

\item $\mathbf{Y}\leftarrow E\rightarrow \mathcal{G}$: Environmental noise $E$ induces distribution shifts in both the input graph $\mathcal{G}$ and the ground-truth label $\mathbf{Y}$.

\item $A\leftarrow \mathcal{G}\rightarrow \mathbf{X}$: The input graph $\mathcal{G}$ comprises the adjacency matrix $A$ (topological structure) and the node feature matrix $\mathbf{X}$ (attribute information).

\item $\mathbf{Z}_a \leftarrow \mathbf{Z}_p \rightarrow \mathbf{Z}_r$: Node classification requires GNNs to aggregate intra-class nodes and separate inter-class nodes. Consequently, the predictive representation $\mathbf{Z}_p$ governs both the intra-class representation $\mathbf{Z}_a$ and the inter-class representation $\mathbf{Z}_r$.

\item $\mathbf{Z}_a \rightarrow \mathbf{Y}$: The intra-class representation $\mathbf{Z}_a$ dictates the ground-truth label $\mathbf{Y}$, as accurate classification relies on the tight aggregating of intra-class nodes within the latent space.

\end{itemize}

Traditional statistical methods evaluate the influence of $\mathbf{Z}_p$ on $\mathbf{Y}$ by calculating the observational conditional probability $P(\mathbf{Y}|\mathbf{Z}_p)$ \cite{10.1145/3534678.3539366}. However, this approach introduces confounding bias, which triggers the \emph{Info-Jitter} phenomenon and degrades the generalization capability of GNNs. The causal graph reveals three backdoor paths: $\mathbf{Z}_r \leftarrow \mathbf{Z}_p \rightarrow \mathbf{Z}_a \rightarrow \mathbf{Y}$, $\mathbf{Z}_p \leftarrow A \leftarrow \mathcal{G} \leftarrow E \rightarrow \mathbf{Y}$, and $\mathbf{Z}_p \leftarrow \mathbf{X} \leftarrow \mathcal{G} \leftarrow E \rightarrow \mathbf{Y}$. These paths originate from fork structures (i.e., a common cause influencing multiple descendants), which open channels for non-causal information flow. Based on the SCM, the estimation of $P(\mathbf{Y}|\mathbf{Z}_p)$ suffers from confounding bias; specifically, two central fork variables induce these spurious correlations \cite{Pearl2014Interpretation}. First, the structure $\mathbf{Z}_r \leftarrow \mathbf{Z}_p \rightarrow \mathbf{Z}_a$ establishes $\mathbf{Z}_p$ as a confounder that induces non-causal flow between the intra-class representation $\mathbf{Z}_a$ and the inter-class representation $\mathbf{Z}_r$. In the context of node classification, $\mathbf{Z}_a$ and $\mathbf{Z}_r$ should inherently exhibit distinct aggregating and separation properties, respectively. This fork structure entangles the two distributions, blurring their boundaries and compromising the discriminative power of the node representations. Second, the environmental noise $E$ acts as an unobserved confounder, injecting noise into the learning process of $\mathbf{Z}_p$ via the input graph $\mathcal{G}$. Driven by these two fork structures, the optimization of GNNs degenerates from capturing invariant causal mechanisms to exploiting spurious correlations. From a causal perspective, this structural vulnerability fundamentally explains both target phenomena. During the training phase, the non-causal associations induced by environmental noise $E$ and the entangled latent distributions ($\mathbf{Z}_a$ and $\mathbf{Z}_r$) are statistically unstable across batches. As the model adjusts its weights to fit these shifting spurious features, the mutual information between the predictive representation $\mathbf{Z}_p$ and the label $\mathbf{Y}$ fluctuates continuously, structurally manifesting as the \emph{Info-Jitter} phenomenon. Furthermore, when deployed in OOD scenarios, the environment-specific correlations memorized via these backdoor paths inevitably break. Lacking causal intra-class representation to anchor its predictions, the model suffers a collapse in its decision boundary, leading to severe generalization degradation \cite{10.1145/3644392}.

To evaluate the true causal effect of $\mathbf{Z}_p$ on $\mathbf{Y}$, we resort to $do$-calculus and employ backdoor adjustment to actively intervene on $\mathbf{Z}_p$, denoted as $P(\mathbf{Y}|do(\mathbf{Z}_p))$. While Randomized Controlled Trials (RCTs) are the gold standard for implementing such interventions, the presence of the unobserved variable $E$ renders physical interventions infeasible. To overcome this limitation, the following theorem demonstrates that $P(\mathbf{Y}|do(\mathbf{Z}_p))$ can be estimated using solely observational data, bypassing the requirement to observe $E$.
\begin{theorem}
\label{theorem3.1}
Given a causal graph in Figure \ref{Figure2}, we can obtain a equation that estimate causal relationships from $\mathbf{Z}_p$ to $\mathbf{Y}$:
\begin{equation}\label{eq_backdoor}
    P_\theta(\mathbf{Y}|do(\mathbf{Z}_p))=\mathbb{E}_{P(\mathcal{G}_v)}\big[ P_\theta(\mathbf{Y}|\mathbf{Z}_p, \mathcal{G}_v)\big],
\end{equation}
where $\theta$ represents the parameter of classifier.
\end{theorem}

The detailed proof is provided Appendix \ref{appendix}. Based on the interventional distribution derived in Eq. \ref{eq_backdoor}, we formulate the causal optimization objective for our model by maximizing its expected log-likelihood:

\begin{equation}\label{opt}
	\mathcal{L}(\theta)= \max_{\theta} \log\mathbb{E}_{P(\mathcal{G}_v)} P_\theta(\mathbf{Y}|\mathbf{Z}_p, \mathcal{G}_v).
\end{equation}
Optimizing Eq. \ref{opt} enforces marginalization over the ego-graph distribution $P(\mathcal{G}_v)$. From a causal perspective, this operation severs the backdoor paths induced by the environmental noise $E$, preventing the GNNs from overfitting to shifting external patterns. This neutralizes the confounding bias from $E$ and provides a foundational step toward resolving the \emph{Info-Jitter} phenomenon. However, addressing external distribution shifts alone is insufficient for robust OOD generalization. The predictive representation $\mathbf{Z}_p$ remains vulnerable to the internal fork structure $\mathbf{Z}_r \leftarrow \mathbf{Z}_p \rightarrow \mathbf{Z}_a$. The marginalization in Eq. \ref{opt} fails to disrupt the non-causal information flow between the causal intra-class representation $\mathbf{Z}_a$ and the spurious inter-class representation $\mathbf{Z}_r$. Leaving these latent distributions entangled inherently leads to unstable decision boundaries across domains. Therefore, isolating pure causal mechanisms and eradicating \emph{Info-Jitter} requires additional latent constraints on $\mathbf{Z}_r$ and $\mathbf{Z}_a$ to decouple the predictive representations.

\subsection{Theoretical Lower Bound}

To operationalize the aforementioned requirement of decoupling $\mathbf{Z}_a$ and $\mathbf{Z}_r$, the model must capture their posterior distributions conditioned on $\mathbf{Z}_p$. However, directly incorporating these latent variables into the log-likelihood of Eq. \ref{opt} requires integrating over the unknown true posteriors $P(\mathbf{Z}_a|\mathbf{Z}_p)$, $P(\mathbf{Z}_r|\mathbf{Z}_p)$, and $P(\mathcal{G}_v)$. This marginalization over the latent space is computationally intractable. 

To bypass this intractability, we resort to variational inference \cite{DBLP:journals/corr/KingmaW13}. Introducing an Evidence Lower Bound (ELBO) not only provides a tractable surrogate objective but also resolves the latent entanglement issue. Specifically, the ELBO yields a set of Kullback-Leibler (KL) divergence terms. These KL terms serve as the exact mathematical regularizers required to break the internal fork structure. By defining parameterized variational distributions $Q_\phi$, we gain the explicit mechanism to constrain $Q_\phi(\mathbf{Z}_a|\mathbf{Z}_p)$ and $Q_\phi(\mathbf{Z}_r|\mathbf{Z}_p)$, forcing them to align with their respective ideal structural properties (i.e., intra-class aggregation and inter-class separation). 

\begin{theorem}\label{theorem4.2}
	Let $Q_\phi(\mathcal{G}_v|\mathbf{Z}_p)$, $Q_\phi(\mathbf{Z}_a|\mathbf{Z}_p)$, and $Q_\phi(\mathbf{Z}_r|\mathbf{Z}_p)$ denote the variational posterior distributions approximating the true distributions $P(\mathcal{G}_v)$, $P(\mathbf{Z}_a|\mathbf{Z}_p)$, and $P(\mathbf{Z}_r|\mathbf{Z}_p)$, respectively. The ELBO for optimizing Eq. \ref{opt} under these structural constraints is formulated as:
	\begin{equation}\label{eq:3}
		\begin{aligned}
			\mathcal{L}(\theta, \phi) \geq \;& \mathbb{E}_{Q_\phi(\mathcal{G}_v|\mathbf{Z}_p)}\mathbb{E}_{Q_\phi(\mathbf{Z}_a|\mathbf{Z}_p)} \big[ \log P_\theta(\mathbf{Y}|\mathbf{Z}_p, \mathbf{Z}_a, \mathcal{G}_v) \big] \\
			&- \text{KL} \big( Q_\phi(\mathcal{G}_v|\mathbf{Z}_p) \parallel P(\mathcal{G}_v) \big) \\
			&- \text{KL} \big( Q_\phi(\mathbf{Z}_a|\mathbf{Z}_p) \parallel P(\mathbf{Z}_a|\mathbf{Z}_p) \big) \\
			&- \text{KL} \big( Q_\phi(\mathbf{Z}_r|\mathbf{Z}_p) \parallel P(\mathbf{Z}_r|\mathbf{Z}_p) \big),
		\end{aligned}
	\end{equation}
	where $\phi$ represents the parameters of the variational posteriors.
\end{theorem}

The detailed proof is provided in Appendix \ref{appendixA}. By maximizing this ELBO, we transform the marginalization into a deterministic objective that can be efficiently optimized via standard backpropagation over parameters $\theta$ and $\phi$. Crucially, the explicit KL divergence terms provide a principled mechanism to enforce the targeted constraints on $\mathbf{Z}_a$ and $\mathbf{Z}_r$. This severs the non-causal flow, isolating the true causal effect of $\mathbf{Z}_p$ on $\mathbf{Y}$. In the subsequent section, we detail how these theoretical regularizers are instantiated into computable loss functions to achieve the ultimate node classification objective.
\subsection{Causal Representation Learning}
To instantiate the optimization objective of Eq. \ref{eq:3}, we first project the ego-graph $\mathcal{G}_v$ into latent space. Specifically, since $\mathcal{G}_v$ comprises both topological structures (the adjacency matrix $A$) and node attributes (the feature matrix $\mathbf{X}$), we process these coupled components through a GNN encoder. This encoding phase yields a joint latent representation $\mathbf{Z}$, which parameterizes the distribution $P(\mathcal{G}_v)$. Building upon this foundational embedding, we now detail the proposed re-weighted representation learning (RRL) module.

To construct the predictive representation $\mathbf{Z}_p$, standard message-passing schemes often indiscriminately aggregate both causal features and spurious environmental noise. To address this, we design a gating mechanism to filter out confounding bias and isolate invariant causal mechanisms. We initialize $\mathbf{Z}^{(0)}_c$ using the output $\mathbf{Z}^{(1)}$ from the first GNN encoding pass, utilizing it as the base representation for subsequent layers. We bypass the direct generation of $\mathbf{Z}_p$ and instead introduce a node-level causal modulation strategy. Before generating the final $\mathbf{Z}_p$, we compute a causal modulation matrix $\Gamma$ from $\mathbf{Z}$ using a Gumbel-softmax function \cite{DBLP:conf/iclr/JangGP17} during the training phrase and softmax during the test phrase. Moreover, inspired by the k-branch in GLIND \cite{wu2024learning}, we adopt a multi-branch encoder, where each branch specializes in learning hierarchical diversified representations and causal mechanisms. This multi-branch design outperforms its single-branch counterpart and improves model robustness against distribution shifts.

\begin{equation}
	\Gamma_{v, k}^{(l)} = \frac{\exp((\mathbf{Z}_{v}^{(l)}+g_k)/\tau)}{\sum_k^K\sum_{u\in \mathcal{V}}\exp((\mathbf{Z}_{u}^{(l)}+g_k)/\tau)},
\end{equation}
\begin{equation}
	g_k=-\log(-\log(u_i)), u_i\sim U(0,1),
\end{equation}
where $\Gamma_{v,k}^{(l)} \in (0,1)$ represents the causal modulation score of node $v$ for the $k$-th causal branch, $g_k$ is a noise following the Gumbel distribution, and $\tau$ controls the degree of discreteness of the distribution. The Gumbel-softmax function normalizes the raw node representations into a probabilistic simplex, inducing a competitive gating mechanism within the ego-graph. From a causal perspective, this competition bounds the information flow: assigning higher confidence scores to nodes that embody invariant causal rationales forces the suppression of nodes dominated by the environment-specific noise $E$. However, to effectively inject this causal modulation matrix $\Gamma$ into the actual forward propagation, we must account for the distinct message-passing paradigms of base models. Therefore, rather than applying a universal but sub-optimal aggregation layer, we structurally augment two canonical GNN architectures to internalize this causal modulation: 

\textbf{CGRL-GCN.} Traditional isotropic diffusion blends ego-semantics with neighborhood contexts. To decouple these signals, we structurally augment the GCN framework \cite{kipf2017semisupervised} into a multi-branch architecture. For a target node $v$, we first compute the structural message $\mathbf{m}_v^{(l)}$ from its neighbors, and concatenate it with its ego-representation to preserve information integrity. The $k$-th causal branch then projects this coupled representation:
\begin{align}
	\mathbf{m}_v^{(l)} &= \sum_{u \in \mathcal{N}(v)} \frac{1}{C_{uv}} \mathbf{z}_{p,u}^{(l)} , \\
	\mathbf{h}_{v,k}^{(l)} &= \Big[ \mathbf{m}_v^{(l)} \Vert \mathbf{z}_{p,v}^{(l)} \Big] \mathbf{W}_p^{(l,k)},
\end{align}
where $C_{uv}=\sqrt{\text{deg}(v)\text{deg}(u)}$, and $[\cdot || \cdot]$ represents the concatenation operation, and $\mathbf{W}_p^{(l,k)} \in \mathbb{R}^{2d \times d}$ is the learnable weight matrix for the $k$-th branch. Finally, we dynamically assemble the predictions via the causal gating scores and apply a residual connection:
\begin{equation}
	\mathbf{Z}_{p}^{(l+1)} = \mathbf{Z}_{p}^{(l)} + \sigma \Bigg(\sum_{k=1}^K \sum_{v \in \mathcal{V}} \Gamma_{v,k}^{(l)} \cdot \mathbf{h}_{v,k}^{(l)} \Bigg),
\end{equation}
where $\sigma$ denotes activation function (e.g., Relu).

\textbf{CGRL-GAT.} Given the anisotropic nature of attention mechanisms, we integrate the causal modulation matrix into the Graph Attention Network (GAT) \cite{veličković2018graph}. Each branch $k$ explicitly models the pairwise causal relevance between connected nodes. We first project the node features via branch-specific weights $\mathbf{W}_{\text{head}}^{(l,k)} \in \mathbb{R}^{d \times d}$:
\begin{align}
	\alpha_{v,u}^{(k)} &= \frac{\exp \Big( \text{LeakyReLU} \big( {\mathbf{a}^{(k)}}^\top [ \mathbf{z}_{p,v}^{(l)}\mathbf{W}_{\text{head}}^{(l,k)} \Vert \mathbf{z}_{p,u}^{(l)} \mathbf{W}_{\text{head}}^{(l,k)} ] \big) \Big)}{\sum_{u' \in \mathcal{N}(v)} \exp \Big( \text{LeakyReLU} \big( {\mathbf{a}^{(k)}}^\top [ \mathbf{z}_{p,v}^{(l)}\mathbf{W}_{\text{head}}^{(l,k)} \Vert \mathbf{z}_{p,u'}^{(l)}\mathbf{W}_{\text{head}}^{(l,k)} ] \big) \Big)},
\end{align}
where $\mathbf{a}^{(k)} \in \mathbb{R}^{2d}$ parameterizes the attention mechanism for the $k$-th branch. After computing the attention-guided messages, the final representation is updated through a causally modulated residual framework:
\begin{equation}
	\mathbf{Z}_{p}^{(l+1)} = \mathbf{Z}_{p}^{(l)} + \sigma \Bigg(\sum_{k=1}^K   \sum_{u \in \mathcal{N}(v)} \Gamma_{v,k}^{(l)} \alpha_{v,u}^{(k)}\mathbf{z}_{p,u}^{(l)}\mathbf{W}_{\text{head}}^{(l,k)} \Bigg).
\end{equation}

Fundamentally, this structural re-weighting can be interpreted as a node-level mechanism for mitigating the vulnerabilities identified in our causal graph. As discussed in Section \ref{section4.1}, the unobserved environment $E$ acts as a confounder through the fork structure $\mathcal{G} \leftarrow E \rightarrow \mathbf{Y}$. This common-cause pathway introduces environment-dependent perturbations into the ego-graph, thereby inducing spurious correlations between the observed graph structure and the ground-truth labels. Traditional GNNs tend to absorb such non-causal patterns during message passing. By learning the causal modulation matrix $\Gamma$, our model applies an adaptive soft-masking mechanism that suppresses the propagation of environment-sensitive signals. Given the iterative nature of GNN message passing, imposing this modulation at each layer progressively reduces the accumulation of spurious information across network depth. As a result, the predictive representation $\mathbf{Z}_p$ is encouraged to focus more strongly on invariant intra-class structure and less on environment-dependent interference. This progressively regularized $\mathbf{Z}_p$ is particularly important for the subsequent energy-based reconstruction of the ego-graph $\mathcal{G}_v$. By providing a more stable and less confounded representation for reconstruction, it helps prevent the repeated amplification of spurious correlations during training.

\subsection{Optimization}\label{section3.4}
Having established the forward causal representation learning module to obtain the predictive representation $\mathbf{Z}_p$, we now detail the optimization strategy. The overall training objective is governed by the theoretical lower bound derived in Theorem \ref{theorem4.2}, which we instantiate into three primary optimization components.

\subsubsection{Latent Disentanglement via Intra-class Aggregation and Inter-class Separation.}\label{section4.4.1}
To instantiate the KL regularization terms in Eq. \ref{eq:3}, we introduce a pair of complementary empirical objectives that operationalize the structural constraints on the latent variables. Specifically, the intra-class aggregation enforces causal invariance within $\mathbf{Z}_a$, while the inter-class separation suppresses spurious correlations encoded in $\mathbf{Z}_r$ from $E$. Together, these objectives provide tractable surrogates for the otherwise intractable KL divergence terms in the latent space.

\textbf{Intra-class Aggregation.} Nodes belonging to the intra-class share underlying causal mechanisms, despite variations in their environmental spurious features \cite{CIA-LRA}. To operationalize this, we aim to align the causal representations $\mathbf{Z}_a$ across intra-class nodes, filtering out environmental noise to distill the invariant causal rationale. We formulate an empirical intra-class alignment objective:
\begin{align}
	\mathcal{L}_{intra} &= \mathbb{E}_{c \in \mathcal{Y}} \mathbb{E}_{i, j \in \mathcal{V}_c} \Big[ \mathcal{D}(z_{a,i}, z_{a,j}) \Big], \label{eq:intra_loss}
\end{align}
where $\mathcal{V}_c$ denotes the set of nodes in class $c$, and $\mathcal{D}$ denotes the Euclidean distance (e.g., $\|\mathbf{z}_{a, i} - \mathbf{z}_{a, j}\|_2^2$). This objective is also qualitatively consistent with the geometric intuition of Neural Collapse \cite{papyan2020prevalence}, where learned features tend to contract towards class means in the over-parameterized regime. To connect this geometric objective with variational inference, we formalize the distributional properties of the latent space through the following assumption and proposition.

\begin{assumption}[Gaussian Prior and Posterior]\label{assumption_causal}
	To establish a tractable variational framework, we adopt an analytical modeling choice: the causal prior conditioned on the class label $c$ is approximated as an isotropic Gaussian anchored at a class prototype $\boldsymbol{\mu}_c$, i.e., $P^{(c)}(\mathbf{Z}_a|\mathbf{Z}_p) = \mathcal{N}(\boldsymbol{\mu}_c, \sigma_p^2 \mathbf{I})$, where $\sigma_p^2 > 0$ regulates the allowable causal dispersion. Furthermore, we assume the class-conditional variational posterior parameterized by the GNNs belong to the Gaussian family $Q^{(c)}_\phi(\mathbf{Z}_a|\mathbf{Z}_p) = \mathcal{N}(\boldsymbol{\mu}_Q^{(c)}, \boldsymbol{\Sigma}_Q^{(c)})$, and its differential entropy is bounded from below ($H(Q_\phi^{(c)}) \geq \xi > -\infty$) to prevent pathological collapse.
\end{assumption}

\begin{proposition}[Surrogate Optimization for Intra-class Variance Penalty]\label{prop_intra}
	Under Assumption \ref{assumption_causal}, assuming representations of nodes within the intra-class are sampled conditionally independently from $Q^{(c)}_\phi(\mathbf{Z}_a|\mathbf{Z}_p)$, the empirical pairwise alignment loss $\mathcal{L}_{intra}$ satisfies the exact identity under the squared Euclidean metric: $\mathcal{L}_{intra} = 2\mathbb{E}_{c \in \mathcal{Y}}[\text{Tr}(\boldsymbol{\Sigma}_Q^{(c)})]$. Consequently, minimizing $\mathcal{L}_{intra}$ serves as a theoretically justified surrogate that directly regulates the expected intra-class covariance trace (i.e., the variance penalty term) within the intractable KL divergence $\mathbb{E}_{c \in \mathcal{Y}}[\text{KL}(Q^{(c)}_\phi(\mathbf{Z}_a|\mathbf{Z}_p) \parallel P^{(c)}(\mathbf{Z}_a|\mathbf{Z}_p))]$.
\end{proposition}
The detailed proof is provided in Appendix \ref{appendixB}. While the intra-class objective directly suppresses the causal dispersion to enforce invariance within $\mathbf{Z}_a$, comprehensive latent disentanglement necessitates a parallel constraint on the subspace $\mathbf{Z}_r$.

\textbf{Inter-class Separation.}
Environmental confounders induce spurious geometric overlap in the latent space, causing nodes from inter-class to become entangled. To mitigate this effect, the model should increase the distributional discrepancy between distinct class-conditional posteriors. However, directly optimizing such divergence in a high-dimensional latent space is generally intractable and suffers from unstable gradients. We therefore adopt a margin-based geometric surrogate to promote inter-class separation:
\begin{align}
	\mathcal{L}_{inter} 
	&= \mathbb{E}_{c \neq c'} \;
	\mathbb{E}_{i \in \mathcal{V}_c,\; j \in \mathcal{V}_{c'}}
	\Big[\max\big(0,\gamma-\mathcal{D}(\mathbf{z}_{r,i},\mathbf{z}_{r,j})\big)\Big],
	\label{eq:inter_loss}
\end{align}
where $\gamma>0$ is a structural margin. As shown below, this objective serves as an empirical surrogate for promoting inter-class Maximum Mean Discrepancy (MMD) in a Reproducing Kernel Hilbert Space (RKHS).

\begin{assumption}[Spurious Overlap as MMD Collapse]\label{assumption_inter}
	Let $Q_c = Q_{\phi}^{(c)}(\mathbf{Z}_r \mid \mathbf{Z}_p)$ and $Q_{c'} = Q_{\phi}^{(c')}(\mathbf{Z}_r \mid \mathbf{Z}_p)$ denote the class-conditional variational posteriors for two distinct classes $c$ and $c'$ in the latent space. The confounded backdoor paths induce topological entanglement, which manifests as a severe reduction of the inter-class MMD between these posteriors, i.e., $\mathrm{MMD}(Q_c,Q_{c'}) \to 0$.

\end{assumption}

\begin{proposition}[Promoting Separation via Surrogate MMD Inflation]\label{prop_inter}
	Let $\mathcal{H}$ be an RKHS endowed with a shift-invariant characteristic kernel $
	k(\mathbf{z}_r,\tilde{\mathbf{z}_r})=\kappa\!\big(\mathcal{D}(\mathbf{z}_r,\tilde{\mathbf{z}_r})\big)$,
	where $\kappa$ is monotonically decreasing and bounded by $\kappa_{\max}$. Then the empirical margin loss $\mathcal{L}_{inter}$ acts as a tractable surrogate for suppressing the cross-distribution kernel expectation. Consequently, minimizing $\mathcal{L}_{inter}$ yields an explicit probability bound on the cross-kernel term, thereby promoting an increase in $\mathrm{MMD}(Q_c,Q_{c'})$ under controlled within-class distributions and empirically improving the decoupling of spurious inter-class overlap.
\end{proposition}

\begin{proof}
	By Gretton et al.~\cite{JMLR:v13:gretton12a}, the squared MMD between $Q_c$ and $Q_{c'}$ can be written as
	\begin{equation}\label{eq:mmd_expand_inter}
		\mathrm{MMD}^2(Q_c,Q_{c'})
		=
		\mathbb{E}_{\mathbf{z}_r,\mathbf{z}_r'\sim Q_c}\!\big[k(\mathbf{z}_r,\mathbf{z}_r')\big]
		+
		\mathbb{E}_{\tilde{\mathbf{z}}_r,\tilde{\mathbf{z}}_r'\sim Q_{c'}}\!\big[k(\tilde{\mathbf{z}}_r,\tilde{\mathbf{z}}_r')\big]
		-
		2\mathbb{E}_{\mathbf{z}_r\sim Q_c,\;\tilde{\mathbf{z}}_r\sim Q_{c'}}\!\big[k(\mathbf{z}_r,\tilde{\mathbf{z}}_r)\big].
	\end{equation}
	To promote the increase of $\mathrm{MMD}^2(Q_c,Q_{c'})$, it is sufficient to suppress the cross-distribution kernel expectation $\Psi_{\mathrm{cross}}	=
	\mathbb{E}_{\mathbf{z}_r\sim Q_c,\;\tilde{\mathbf{z}}_r\sim Q_{c'}}\!\big[k(\mathbf{z}_r,\tilde{\mathbf{z}}_r)\big]$. We use $\mathcal{L}_{inter}$ as a computationally efficient geometric surrogate for this purpose. Let $ \xi=\max\big(0,\gamma-\mathcal{D}(\mathbf{z}_r,\tilde{\mathbf{z}}_r)\big)\ge 0$, where its expectation is taken over the empirical inter-class pair sampling distribution induced by the training scheme, satisfying $\mathbb{E}[\mathcal{\xi}] = \mathcal{L}_{inter}$. Then for any fixed tolerance $\delta\in(0,\gamma)$, Markov's inequality gives
	\begin{equation}\label{eq:markov}
		\Pr\!\big(\mathcal{D}(\mathbf{z}_r,\tilde{\mathbf{z}}_r)\le \gamma-\delta\big)
		=
		\Pr(\xi\ge \delta)
		\le
		\frac{\mathbb{E}[\xi]}{\delta}
		=
		\frac{\mathcal{L}_{inter}}{\delta}.
	\end{equation}
	
	Since the kernel is bounded by $\kappa_{\max}$ and monotonically decreasing with distance, we can bound the cross-kernel expectation by decomposing the latent space into two disjoint regions:
	\begin{align}\label{eq:kernel_bound}
		\Psi_{\mathrm{cross}}
		&\le
		\kappa(\gamma-\delta)\cdot
		\Pr\!\big(\mathcal{D}(\mathbf{z}_r,\tilde{\mathbf{z}}_r)>\gamma-\delta\big)
		+
		\kappa_{\max}\cdot
		\Pr\!\big(\mathcal{D}(\mathbf{z}_r,\tilde{\mathbf{z}}_r)\le\gamma-\delta\big)
		\nonumber\\
		&\le
		\kappa(\gamma-\delta)
		+
		\kappa_{\max}\frac{\mathcal{L}_{inter}}{\delta},
	\end{align}
	where the residual term $\varepsilon = \kappa_{max} \frac{\mathcal{L}_{inter}}{\delta}$ is proportional to the residual probability mass of margin-violating pairs. For any fixed tolerance $\delta$, this residual vanishes as the empirical surrogate $\mathcal{L}_{inter} \to 0$. By substituting Eq. \ref{eq:kernel_bound} into Eq. \ref{eq:mmd_expand_inter}, the negative penalty term $-2\Psi_{cross}$ is bounded from below by $-2(\kappa(\gamma - \delta) + \varepsilon)$. 

	Assuming that the within-class kernel expectations in Eq.~\eqref{eq:mmd_expand_inter} are controlled or promoted by the intra-class aggregation objective $\mathcal{L}_{intra}$, minimizing $\mathcal{L}_{inter}$ suppresses cross-distribution overlap by shrinking $\varepsilon$. This, in turn, promotes an increase in the inter-class MMD and thereby empirically mitigates spurious inter-class entanglement.
\end{proof}

\begin{remark}[Information-Theoretic Regularization via a Geometric Proxy]\label{remark_kl}
	The geometric separation induced by $\mathcal{L}_{inter}$ is also consistent with an information-theoretic view. Under the bounded-kernel assumption, MMD is upper-bounded by the scaled total variation distance, and Pinsker's inequality further yields
	\begin{equation}
		\mathrm{KL}(Q_c\parallel Q_{c'})
		\ge
		\frac{1}{2\kappa_{\max}}
		\mathrm{MMD}^2(Q_c,Q_{c'}).
	\end{equation}
	Therefore, increasing inter-class MMD through $\mathcal{L}_{inter}$ also increases a lower bound on the corresponding inter-class KL divergence. From a representation perspective, the shared topological overlap between $Q_c$ and $Q_{c'}$ can be viewed as a carrier space for the spurious component $\mathbf{Z}_r$. In this sense, enlarging the inter-class discrepancy provides a geometric regularization signal that is consistent with the variational penalty $\mathrm{KL}\!\big(Q_\phi(\mathbf{Z}_r\mid \mathbf{Z}_p)\parallel P(\mathbf{Z}_r\mid \mathbf{Z}_p)\big)$
	toward an uninformative prior. Although a bounded kernel does not imply mutually singular class-conditional posteriors, pushing this lower bound upward may reduce the model's ability to exploit shared environmental noise, thereby promoting latent disentanglement.
\end{remark}

\subsubsection{Energy-based Reconstruction}

To evaluate the structural KL divergence term derived in Theorem \ref{theorem4.2}, we parameterize the posterior distribution $Q_\phi(\mathcal{G}_v|\mathbf{Z}_p)$. During the forward pass, the input ego-graph $\mathcal{G}_v$ integrates the raw adjacency matrix and node attributes, which the GNN encodes into a latent prior distribution $P(\mathcal{G}_v)$. Because these raw inputs are inherently corrupted by the environmental confounder $E$, this unconstrained prior inevitably absorbs spurious topological noise. Concurrently, the predictive representation $\mathbf{Z}_p$ undergoes continuous deconfounding throughout the causal representation learning and joint optimization process. By employing this purified $\mathbf{Z}_p$ to inversely reconstruct the ego-graph, we obtain the posterior $Q_\phi(\mathcal{G}_v|\mathbf{Z}_p)$, which establishes a mathematically consistent generative loop. We propose an energy-based reconstruction module to instantiate the KL divergence regularizer between $Q_\phi$ and $P$. Minimizing this divergence directly forces the latent space to discard spurious dependencies and encode invariant causal mechanisms, fortifying the model against distribution shifts.

To instantiate this, we first formulate an energy-based model to evaluate the causal affinity between interconnected nodes. For any given edge $\langle u,v \rangle$ in the ego-graph, we define the interaction energy $E(u,v)$ as
\begin{align}
	E(u,v) &= - \mathbf{Z}_{c,v} \mathbf{W}_{uv} \mathbf{Z}_{c,u}^\top, \\
	Q_\phi(u,v) &= \frac{\exp(-E(u,v))}{\sum_{u' \in \mathcal{N}(v)} \exp(-E(u',v))},
\end{align}
where $\mathbf{W}_{uv} \in \mathbb{R}^{d \times d}$ denotes a learnable weight matrix. The energy function $E(u,v)$ quantifies the structural compatibility of the edge based on causal semantics. A lower energy state dictates a higher probability of a true causal connection existing between node $u$ and $v$ within the ego-graph $\mathcal{G}_v$. Furthermore, to establish a differentiable sampling mechanism for the posterior graph distribution during training, we also apply the Gumbel-softmax \cite{DBLP:conf/iclr/JangGP17} over the latent node representations:
\begin{equation}
	Q_\phi(\mathcal{G}_v) = \frac{\exp \big( (\mathbf{Z}_{c, v}^{(l)}+g_k)/\tau \big)}{\sum_k^K\sum_{u\in \mathcal{V}} \exp \big( (\mathbf{Z}_{c, u}^{(l)}+g_{k})/\tau \big)}.
\end{equation}
Consequently, we formulate the joint reconstructed posterior as $Q_\phi(\mathcal{G}_v|\mathbf{Z}_p) = [Q_\phi(\mathcal{G}_v) \Vert Q_\phi(u,v)]$. While traditional energy-based models suffer from intractable partition functions, this directed reconstruction strategy enables the model to execute stable end-to-end gradient descent. Finally, this reconstruction paradigm allows us to instantiate the second term from Eq. \ref{eq:3}:
\begin{equation}
	\mathcal{L}_{rec}(\theta, \phi) = \min \text{KL} \Big( Q_\phi(\mathcal{G}_v|\mathbf{Z}_p) \parallel P(\mathcal{G}_v) \Big) = \min \sum_{\mathcal{G}_v} Q_\phi(\mathcal{G}_v|\mathbf{Z}_p) \log \frac{Q_\phi(\mathcal{G}_v|\mathbf{Z}_p)}{P(\mathcal{G}_v)}.
\end{equation}
By minimizing $\mathcal{L}_{rec}$, we constrain the estimation of the causal graph. This prevents the model from hallucinating spurious topological connections fitting the environmental noise $E$, thereby furnishing a robust, unconfounded topological foundation for downstream node classification.

\subsubsection{Prediction Loss}
The first term in the variational objective (Eq.~\ref{eq:3}) corresponds to the predictive log-likelihood $\mathbb{E}_{Q_\phi(\mathcal{G}_v|\mathbf{Z}_p)}
	\mathbb{E}_{Q_\phi(\mathbf{Z}_a|\mathbf{Z}_p)}
	\big[\log P_\theta(\mathbf{Y}\mid \mathbf{Z}_p,\mathbf{Z}_a,\mathcal{G}_v)\big].$
In our framework, the conditional topology distribution $Q_\phi(\mathcal{G}_v|\mathbf{Z}_p)$ is modeled through the energy-based graph reconstruction module described above. Moreover, the latent representation $\mathbf{Z}_p$ is jointly regularized by the intra-class aggregation and inter-class separation objectives, which encourage compact class-conditional structure while suppressing inter-class overlap. Under this representation-learning scheme, we instantiate the predictive term using a discriminative approximation parameterized only by $\mathbf{Z}_p$. We then parameterize $P_\theta(\mathbf{Y}\mid \mathbf{Z}_p)$ with a classifier
$h_\theta$,
which maps the learned representation to the class-probability simplex and produces the prediction $\hat{\mathbf{Y}}_v = h_\theta(\mathbf{z}_{p,v}).$ Since maximizing the log-likelihood of a categorical distribution is equivalent to minimizing the cross-entropy loss, the first variational term is implemented as the standard supervised classification objective:
\begin{equation}\label{eq:loss_sup}
	\mathcal{L}_{sup}
	=
	-\frac{1}{|\mathcal{V}_{tr}|}
	\sum_{v\in\mathcal{V}_{tr}}
	\sum_{c=1}^{C}
	\mathbf{Y}_{v,c}\log \hat{\mathbf{Y}}_{v,c},
\end{equation}
where $\mathcal{V}_{tr}$ denotes the set of labeled training nodes, $C$ is the number of classes, and $\mathbf{Y}_{v,c}\in\{0,1\}$ is the one-hot encoded ground-truth label of node $v$. This supervised term preserves the label signal while remaining consistent with the variational objective through a tractable discriminative approximation.

Finally, having formulated the individual variational bounds, we consolidate them to establish the complete empirical training objective. By combining the four aforementioned loss terms, the overall optimization formulation is given as follows:
\begin{equation}
	\mathcal{L}(\theta;\phi)= \mathcal{L}_{sup} + \mathcal{L}_{rec} + \lambda_1 \mathcal{L}_{intra} + \lambda_2 \mathcal{L}_{inter},
\end{equation}
where $\lambda_1$ and $\lambda_2$ are the loss weights.

\subsubsection{Scalable Optimization of Pairwise Regularization}
\label{section4.4.4}

The original formulations of $\mathcal{L}_{intra}$ and $\mathcal{L}_{inter}$ in Section~\ref{section4.4.1} are defined over pairwise node relations and therefore incur quadratic cost under naive full-graph evaluation. To improve scalability on large graphs, we adopt a practical implementation that preserves the optimization role of these objectives while significantly reducing their computational overhead.

For the intra-class objective, instead of enumerating all intra-class node pairs, we approximate class-wise compactness by measuring the deviation of each node representation from a class prototype in the latent space $\mathbf{Z}_a$. Let $\mathbf{z}_{a,i}\in\mathbb{R}^{d}$ denote the normalized intra-class representation of node $i$, and let $y_i\in\mathcal{Y}$ be its class label. We define the normalized prototype of class $c$ as
\begin{equation}
	\bar{\mathbf{z}}_{a,c}
	=
	\frac{\sum_{j \in \mathcal{V},\, y_j = c} \mathbf{z}_{a,j}}
	{\left\|\sum_{j \in \mathcal{V},\, y_j = c} \mathbf{z}_{a,j}\right\|_2},
\end{equation}
\begin{equation}
	\mathcal{L}_{intra}^{\mathrm{sc}}
	=
	\frac{1}{|\mathcal{V}|}
	\sum_{i \in \mathcal{V}}
	\mathcal{D}\!\big(\mathbf{z}_{a,i}, \bar{\mathbf{z}}_{a,y_i}\big),
\end{equation}
which encourages each node to align with the prototype of its own class, and $\mathcal{D}$' per-sample computation remains linear in the representation dimension $d$.

For the inter-class objective, instead of enumerating all inter-class node pairs, we adopt random negative sampling in the latent space $\mathbf{Z}_r$. Specifically, for each node $i$, we sample a fixed number of candidate negatives and retain only those belonging to inter-class, denoted by $\mathcal{S}_i^-$. The scalable inter-class surrogate is defined as
\begin{equation}
	\mathcal{L}_{inter}^{\mathrm{sc}}
	=
	\frac{1}{|N|}
	\sum_{i \in \mathcal{V}}
	\frac{1}{|\mathcal{S}_i^-|}
	\sum_{j \in \mathcal{S}_i^-}
	\max \Big(0, \gamma - \mathcal{D}(\mathbf{z}_{r,i}, \mathbf{z}_{r,j}) \Big),
\end{equation}
where $\gamma>0$ is the margin and $\mathcal{S}_i^-$ denotes the set of sampled inter-class nodes associated with node $i$. This objective penalizes inter-class samples that remain too close in the latent space, thereby suppressing spurious inter-class overlap while avoiding quadratic pairwise computation. The resulting scalable regularizer is
\begin{equation}
	\mathcal{L}(\theta;\phi)
	=
	\mathcal{L}_{sup}+ \mathcal{L}_{rec}+
	\lambda_1 \mathcal{L}_{intra}^{\mathrm{sc}}
	+
	\lambda_2 \mathcal{L}_{inter}^{\mathrm{sc}}.
\end{equation}

Overall, $\mathcal{L}_{intra}^{\mathrm{sc}}$ provides a prototype-based relaxation of class-wise compactness, whereas $\mathcal{L}_{inter}^{\mathrm{sc}}$ provides a sampled relaxation of inter-class separation. From a theoretical perspective, the propositions in Section~\ref{section4.4.1} are established for the original pairwise objectives. The scalable implementation introduced here should therefore be understood as a practical relaxation for efficient optimization on large graphs, rather than a replacement of the underlying theoretical formulation.

\subsection{Discussion}
In this section, we further discuss the proposed framework from two perspectives. We provide a domain-adaptation view to interpret its OOD generalization mechanism, and then present computational complexity analysis of CGRL.
\subsubsection{A Domain-Adaptation Perspective for OOD Generalization.} 
To intuitively understand why our dual-objective framework improves OOD generalization, we can interpret the target risk through the lens of standard domain adaptation theory \cite{ben-david}. Using our decoupled representation space $\mathbf{Z}_p = [\mathbf{Z}_a \parallel \mathbf{Z}_r]$ as an interpretive lens, the standard domain-adaptation bound suggests the following representation-aware decomposition of OOD risk:
\begin{equation}\label{eq:generalization_bound}
	\mathcal{R}_{\mathcal{T}}(f) \leq \hat{\mathcal{R}}_{\mathcal{S}}(f) + \frac{1}{2} \sup_{c \in \mathcal{Y}} d_{\mathcal{H}\Delta\mathcal{H}}(\mathcal{S}_c^{\mathbf{Z}_a}, \mathcal{T}_c^{\mathbf{Z}_a}) + \Omega_{\mathcal{H}}(f, \mathbf{Z}_r) + \lambda + \epsilon,
\end{equation}
where $\mathcal{R}_{\mathcal{T}}(f)$ and $\hat{\mathcal{R}}_{\mathcal{S}}(f)$ denote the expected target risk and the empirical source risk, respectively. Crucially, $\Omega_{\mathcal{H}}(f, \mathbf{Z}_r)$ denotes an abstract non-negative functional capturing the excess target risk induced by the classifier's reliance on spurious features. The terms $\lambda$ and $\epsilon$ represent the ideal joint risk and the hypothesis complexity constant. As in standard domain adaptation theory, this perspective is meaningful when the ideal joint risk $\lambda$ is not dominant. Furthermore, the class-conditional discrepancy term $\sup_{c \in \mathcal{Y}} d_{\mathcal{H}\Delta\mathcal{H}}(\mathcal{S}_c^{\mathbf{Z}_a}, \mathcal{T}_c^{\mathbf{Z}_a})$ should be understood as a task-specific, representation-aware refinement of the standard $\mathcal{H}\Delta\mathcal{H}$-divergence perspective, tailored for node classification.

This conceptual decomposition helps clarify the distinct, yet complementary, roles of our proposed structural objectives. First, the intra-class alignment objective $\mathcal{L}_{intra}$ reduces the class-conditional dispersion in $\mathbf{Z}_a$. This geometric compactness promotes cross-environment stability, thereby empirically restricting the worst-case causal divergence term $\sup_{c \in \mathcal{Y}} d_{\mathcal{H}\Delta\mathcal{H}}(\mathcal{S}_c^{\mathbf{Z}_a}, \mathcal{T}_c^{\mathbf{Z}_a})$. Second, addressing the inherently unbounded environmental shifts in $\mathbf{Z}_r$, the inter-class separation objective $\mathcal{L}_{inter}$ suppresses the inter-class topological overlap. Conceptually related to the label-predictive capacity of $\mathbf{Z}_r$, this suppression provides a regularization signal that functionally discourages the classifier from extracting shared environmental noise. In practice, this explicitly mitigates the spurious sensitivity penalty $\Omega_{\mathcal{H}}(f, \mathbf{Z}_r)$.

Together, this domain-adaptation-inspired view is intended as an interpretive framework rather than a tight, directly estimable bound. Under the usual adaptation assumption that the ideal joint risk $\lambda$ is not dominant, rather than attempting the impossible task of ensuring the spurious features $\mathbf{Z}_r$ remain invariant across domains, our two structural objectives help explain the observed OOD gains by promoting cross-environment stability in $\mathbf{Z}_a$ and systematically reducing predictive reliance on spurious components in $\mathbf{Z}_r$.

While recent graph OOD methods typically motivate generalization through cross-environment learning or SCM \cite{CIA-LRA, CaNet}, our framework admits a complementary interpretation from the perspective of domain adaptation. Rather than depending on environment annotations or strong invariance assumptions, it promotes OOD robustness by reducing the class-conditional causal discrepancy $\sup_{c \in \mathcal{Y}} d_{\mathcal{H}\Delta\mathcal{H}}$ via intra-class aggregation and mitigating spurious predictive sensitivity $\Omega_{\mathcal{H}}$ via inter-class separation. This provides a tractable, representation-centric explanation for generalization under distribution shifts.

\subsubsection{Time Complexity Analysis}\label{sec:time_complexity}
We analyze the training complexity of the proposed framework for node classification on a graph with $|\mathcal{V}|$ nodes and $|\mathcal{E}|$ edges. Let $L$ denote the number of GNN layers, $d$ the hidden dimension, $K$ the number of branches, and $S$ the number of sampled negatives per node. For the multi-branch causal encoding module, each branch performs feature transformation and message passing over the input graph, leading to a complexity of $\mathcal{O}\!\left(|\mathcal{V}|d^2 + |\mathcal{E}|d\right)$. For the energy-based reconstruction module, computing edge compatibility scores and the associated latent sampling incurs
$
\mathcal{O}\!\left(|\mathcal{E}|d^2 + K|\mathcal{V}|d\right)$. For the original pairwise regularizers, naive full-graph evaluation requires
$
\mathcal{O}\!\left(
\sum_{c\in\mathcal{Y}}|\mathcal{V}_c|^2 d
+
\sum_{c\neq c'}|\mathcal{V}_c||\mathcal{V}_{c'}|d
\right)$,
which is upper bounded by $\mathcal{O}(|\mathcal{V}|^2 d)$. Therefore, the overall naive training complexity is
$
\mathcal{O}\!\left(|\mathcal{V}|d^2 + |\mathcal{E}|d
+ |\mathcal{E}|d^2
+ |\mathcal{V}|^2 d
\right)$.

In practice, for large-scale graphs we adopt the scalable implementation described above. Under this implementation, the intra-class prototype computation requires $\mathcal{O}(|\mathcal{V}|d)$, while the sampled inter-class regularization requires $\mathcal{O}(|\mathcal{V}|Sd)$. The resulting practical complexity becomes
$
\mathcal{O}\!\left(|\mathcal{V}|d^2 + |\mathcal{E}|d
+ |\mathcal{E}|d^2
+ |\mathcal{V}|d
+ |\mathcal{V}|Sd
\right)$.
When $L$, $K$, $d$, and $S$ are treated as constants, the graph-dependent computation is effectively near-linear in graph size.

\section{EXPERIMENTS}

In this section, to verify the generalization performance of our proposed CGRL framework, we conduct experiments on $8$ open-source datasets with public split 






\begin{table}[h]
\caption{Datasets with feature and spatiotemporal shifts.}
\label{table4}
\centering
\begin{tabular}{cccccc}
\toprule
Datasets & Nodes & Edges & Classes&Features&Shift types\\
\midrule
Cora    &2708 &5429 &7 &1433 &Feature\\
Citeseer &3327 &4732 &6 &3703 &Feature\\
Pubmed &19717 &44338 &3 &500 &Feature\\
Arxiv &169343 &1166243 &40 &128 &Temporal\\
Twitch &34120 &892346 &2 &2545 &Spatial\\
\bottomrule
\end{tabular}
\end{table}

\subsection{Experimental Settings}
\textbf{Datasets \& Metrics.} Accuracy is employed for Cora, Citeseer, Pubmed \cite{synthetic}, Arxiv \cite{NEURIPS2020_fb60d411}, GOODCora, GOODCBAS and GOODWebKB \cite{gui2022good}. ROC-AUC is used for Twitch \cite{Twitch}.
\begin{itemize}
\item \textbf{datasets with feature and spatiotemporal shifts.} All in-distribution (ID) data are randomly divided with a 50\%/25\%/25\% ratio for training, validation, and test in this datasets. The overall introduction of datasets with feature and spatiotemporal shifts is showed in Table \ref{table4}.
	\begin{itemize}
	\item \textbf{Cora, Citeseer and Pubmed.} Cora, Citeseer, and Pubmed are all well-established homogeneous citation network datasets. Specifically, the first two datasets focus on the computer science domain, while the last one targets biomedical research. In these datasets, nodes correspond to academic papers, and edges denote citation relationships between them. Node features are derived from the bag-of-words model applied to the textual content of the papers. The task on these three datasets is to classify the research topics of papers based on both node features and graph topological structures. We follow the domain partitioning protocol of CaNet \cite{CaNet}. We generate synthetic representations for original nodes to simulate feature distribution shifts while preserving the ground-truth node labels. These representations are further partitioned into six domains, among which the last three are designated as OOD data. This experimental setup is designed to evaluate the robustness of models under feature shifts.
	
	\item \textbf{Arxiv.} Arxiv dataset is a temporally aware citation network of papers in the computer science domain, where nodes represent academic papers and edges denote citation relationships. Its task is to predict the research topic categories of papers, which naturally captures temporal distribution shifts. Compared with classic small-scale datasets such as Cora, it exhibits substantial improvements in scale, complexity, and realism, making it a pivotal benchmark for evaluating the scalability, temporal generalization capability, and OOD generalization capability of GNNs. We define papers published before 2014 as ID data, while those published after 2014 are designated as OOD data. Furthermore, the OOD data are partitioned into three subsets, namely 2014-2015, 2016-2017, and 2018-2020 (with both upper and lower bounds included) \cite{CaNet}.
	
	\item\textbf{Twitch.} Twitch dataset is a well-recognized multilingual social network benchmark, where nodes represent users on the Twitch platform and edges denote mutual follow relationships between users. It consists of 6 cross-lingual subgraphs (DE, EN, ES, FR, PT, RU), which share the same feature space but exhibit distinct distributional characteristics. This renders the dataset well-suited for research on cross-domain transfer and OOD generalization. We designate the subgraphs corresponding to DE, PT, and RU as ID data, while those corresponding to ES, FR, and EN are defined as OOD data \cite{CaNet}.
	\end{itemize}

\item \textbf{Datasets with covariate and concept shifts.}
GOODCora, GOODCBAS, and GOODWebKB are three representative datasets from the Graph Out-of-Distribution (GOOD) benchmark \cite{gui2022good}, which was proposed to address the long-standing issues of inconsistent evaluation protocols and poorly controlled shift types in graph OOD generalization. All three datasets are adapted from publicly available graph benchmarks and retain the original node classification task, while introducing controlled distribution shifts for systematic evaluation. Following \cite{gui2022good}, these datasets cover two major shift categories, namely covariate shifts and concept shifts, and thus provide standardized testbeds for evaluating the robustness of causal GNNs and graph deconfounding methods. Specifically, each dataset is further partitioned into multiple domains according to the predefined shift criterion. For GOODCora, the domains are constructed based on node degrees and word distributions, where the latter is determined by the occurrence frequency of selected vocabulary in publication texts. For GOODCBAS, the domains are defined by color attributes, and for GOODWebKB, they are defined according to university identities. Such a domain construction protocol enables controlled OOD evaluation under diverse graph shifts while preserving the semantic structure of the original node classification task.
\end{itemize}
To quantify the \emph{Info-Jitter} phenomenon, we monitor the mutual information (MI) between predictive representations and ground-truth labels throughout training. At epoch $t$, let $\mathbf{Z}_{p}^{(t)}=\{\mathbf{z}_{p,v}^{(t)}\}_{v\in\mathcal{V}}
$ denote the predictive representations of all nodes. We compute the MI value at epoch $t$ as
$\mathrm{MI}_t = I\!\left(\mathbf{Z}_{p}^{(t)};\mathbf{Y}\right),
$ where $I(\cdot;\cdot)$ denotes the MI estimator applied consistently across all methods and datasets. In other words, the MI trajectory is measured using the predictive representations of all nodes and their labels at each training epoch. While the absolute MI value reflects how much class-relevant information is preserved in the learned representation, our main interest is its temporal stability under OOD shifts. To this end, we focus on the second half of training and define two complementary metrics. First, we define the \emph{Jitter} score as the cumulative absolute variation of the MI trajectory:
\begin{equation}
	\mathrm{Jitter}
	=
	\sum_{t=t_0+1}^{T}
	\left|
	\mathrm{MI}_t-\mathrm{MI}_{t-1}
	\right|,
\end{equation}
where $T$ is the total number of training epochs and $t_0$ denotes the starting epoch of the training (e.g., $t_0=100$). A smaller Jitter score indicates weaker cumulative fluctuation of the MI trajectory and therefore a weaker \emph{Info-Jitter} phenomenon. Second, to evaluate whether the model preserves useful label-relevant information in the stable training phase, we define the \emph{TailMI} score as the average MI over the same interval:
\begin{equation}
	\mathrm{TailMI}
	=
	\frac{1}{T-t_0+1}
	\sum_{t=t_0}^{T}
	\mathrm{MI}_t.
\end{equation}
A larger TailMI score indicates that the model retains stronger class-relevant information in the predictive representation during the later stage of optimization. Together, Jitter and TailMI provide a complementary characterization of the training dynamics: Jitter measures the stability of representation--label alignment, whereas TailMI measures its information content. In the experiments, we report both metrics to assess whether different methods can alleviate the \emph{Info-Jitter} phenomenon while maintaining discriminative predictive representations.

\textbf{Baselines.} To ensure a fair and comprehensive comparison, we instantiate all competing methods on two widely used GNN backbones, namely GCN and GAT. 
\begin{itemize}
	\item \textbf{ERM} learns by minimizing the average empirical risk under the standard i.i.d. assumption, and serves as the conventional training baseline.
	
	\item \textbf{IRM} \cite{IRM} seeks invariant predictors by encouraging the model to rely on features whose predictive effect remains stable across environments.
	
	\item \textbf{DeepCoral} \cite{deepcoral} reduces domain discrepancy by aligning second-order feature statistics between different domains.
	
	\item \textbf{DANN} \cite{DANN} performs adversarial domain adaptation to learn domain-invariant representations shared by source and target distributions.
	
	\item \textbf{GroupDRO} \cite{groupdro} improves worst-group robustness by reweighting the training objective toward groups with larger losses.
	
	\item \textbf{Mixup} \cite{mixup} regularizes the model through interpolation in the feature-label space, and is widely used to improve generalization.
	
	\item \textbf{SRGNN} \cite{zhu2021shiftrobust} is a shift-robust graph neural network designed to mitigate training-test distribution inconsistency in graph data.
	
	\item \textbf{EERM} \cite{wu2022handling} extends invariant learning to graphs by constructing virtual environments from a single observed graph.
	
	\item \textbf{CaNet} \cite{CaNet} introduces a causal node-level framework that attributes OOD degradation to latent confounding bias and infers pseudo-environments without explicit environment labels.
	
	\item \textbf{CIA-LRA} \cite{CIA-LRA} analyzes why classical invariance methods are less effective on graphs and improves robustness by selectively aligning invariant node representations through adjacent-label distributions.
\end{itemize}
\textbf{Implement Details.}
All experiments are conducted on a server running Ubuntu 22.04 with an Intel(R) Xeon(R) Platinum 8470Q CPU (52 Cores @ 2.10GHz), 48 GB of RAM, and a RTX 5090 GPU (32 GB of RAM in total). Furthermore, CGRL is implemented using PyTorch 2.5.1, Python 3.12, and PyTorch Geometric 2.6.1. During training, Adam with weight decay and dropout are introduced to prevent overfitting and enhance the model's generalization performance. For datasets with feature and spatiotemporal shifts, all experiments are conducted over five independent runs, with 500 epochs per run for each dataset. For the large-scale datasets Arxiv and Pubmed, we use the scalable implementation of the pairwise regularizers described in Section \ref{section4.4.4}, where intra-class compactness is approximated by class prototypes and inter-class separation is optimized via random negative sampling. For the remaining datasets, we use the original pairwise implementation. This implementation choice only affects the computation of the regularization terms and does not modify the model architecture or the theoretical formulation. The parameters of CGRL are reinitialized prior to each run, and the best performance on the validation set is recorded upon completion of each run. The final results are obtained by averaging the performance metrics across the five runs with different random seeds. For datasets with covariate and concept shifts, the protocol for selecting results is identical to that used for datasets with feature and spatiotemporal shifts. The maximum number of epochs is set as specified by the GOOD benchmark. Table \ref{para} presents the detailed experimental parameters.

\begin{table}[!ht]
	\centering
	\caption{Parameter settings of CGRL}
	\label{para}
	\resizebox{\textwidth}{!}{
		\begin{tabular}{cccccccc}
			\toprule
			Shifts Type&Datasets& Weight decay & Dropout & Learning rate&  Branches & Hidden dimension & Layer of model  \\ \midrule
			\multirow{3}{*}{Feature} & Cora& 1e-5 & 0.4 & 0.001 & 3 &128 & 4 \\
			& Citeseer& 5e-5 & 0.4 & 0.001& 3& 128 &3 \\
			& Pumbed& 5e-5 &0.3 & 0.01& 2&128 & 2 \\ \midrule
			\multirow{2}{*}{Sptiotemporal}& Arxiv& 5e-5 &0.2 & 0.01& 2 &32 & 2 \\
			& Twitch& 3e-5 &0.1 & 0.01& 4&64 & 5 \\ \midrule
			\multirow{3}{*}{Covariate} & GOODCora & 1e-3& 0 & 1e-3& 3& 256& 4 \\
			& GOODCBAS & 1e-3& 0& 1e-3& 4& 256& 4 \\
			& GOODWebKB & 1e-3& 0.1& 3e-3& 3& 128& 5 \\ \midrule
			\multirow{3}{*}{Concept} & GOODCora& 3e-3& 0& 1e-3& 4& 256& 4 \\
			& GOODCBAS & 1e-3& 0& 3e-3& 4& 256& 4 \\
			& GOODWebKB & 1e-3& 0.1& 5e-3& 4& 128& 5 \\
			\bottomrule
			
	\end{tabular}}
\end{table}

\begin{table*}[htbp]

    \centering
    \caption{Test (mean±standard deviation) Accuracy (\%)  and ROC-AUC (\%) on OOD data of Arxiv and Twitch datasets, respectively. OOM denotes out-of-memory. \textbf{Bold} numbers indicate the best result, and the \underline{underlined} numbers indicate the second-best result.}
    \label{table1}
\resizebox{\textwidth}{!}{
\begin{tabular}{ccccccccccc}

\toprule[1pt]

\multirow{2}{*}{\textbf{Backbone}} & \multirow{2}{*}{\textbf{Method}} & \multirow{2}{*}{\textbf{Cora}} & \multirow{2}{*}{\textbf{Citeseer}}&\multirow{2}{*}{\textbf{Pubmed}}&\multicolumn{3}{c}{\textbf{Arxiv}}&\multicolumn{3}{c}{\textbf{Twitch}}\\
\cmidrule(lr){6-8} \cmidrule(lr){9-11}
&&& && \textbf{2014-2015} & \textbf{2016-2017}&\textbf{2018-2020} & \textbf{ES} & \textbf{FR} & \textbf{EN}\\ \midrule
\multirow{11}{*}{GCN} & ERM &74.30\footnotesize±2.66&74.93\footnotesize±2.39&81.36\footnotesize±1.78& 56.33\footnotesize±0.17 & 53.53\footnotesize±0.44 & 45.83\footnotesize±0.47 & 66.07\footnotesize±0.14 & 52.62\footnotesize±0.01 & 63.15\footnotesize±0.08 \\
 & IRM &74.19\footnotesize±2.60&75.34\footnotesize±1.61&81.14\footnotesize±1.72& 55.92\footnotesize±0.24 & 53.25\footnotesize±0.49 & 45.66\footnotesize±0.83 & 66.95\footnotesize±0.27 & 52.53\footnotesize±0.02 & 62.91\footnotesize±0.08 \\
 & Coral &74.26\footnotesize±2.28&74.97\footnotesize±2.53&81.56\footnotesize±2.35& 56.42\footnotesize±0.26 & 53.53\footnotesize±0.54 & 45.92\footnotesize±0.52 & 66.15\footnotesize±0.14 & 52.67\footnotesize±0.02 & 63.18\footnotesize±0.03 \\
 & DANN &73.09\footnotesize±3.25&74.74\footnotesize±2.78&80.77\footnotesize±1.43& 56.35\footnotesize±0.11 & 53.81\footnotesize±0.33 & 45.89\footnotesize±0.37 & 66.15\footnotesize±0.13 & 52.66\footnotesize±0.02 & 63.20\footnotesize±0.06 \\
 & GroupDRO & 74.25\footnotesize±2.61&75.02\footnotesize±2.05&81.07\footnotesize±1.89& 56.52\footnotesize±0.27 & 53.40\footnotesize±0.29 & 45.76\footnotesize±0.59 & 66.82\footnotesize±0.26 & 52.69\footnotesize±0.02 & 62.95\footnotesize±0.11 \\
 & Mixup &92.77\footnotesize±1.27&77.28\footnotesize±5.28&79.76\footnotesize±4.44& 56.67\footnotesize±0.46 & 54.02\footnotesize±0.51 & 46.09\footnotesize±0.58 & 65.76\footnotesize±0.30 & 52.78\footnotesize±0.04 & 63.15\footnotesize±0.08 \\
 & SRGNN &81.91\footnotesize±2.64&76.10\footnotesize±4.04&84.75\footnotesize±2.38& 56.79\footnotesize±1.35 & 54.33\footnotesize±1.78 & 46.24\footnotesize±1.90 & 65.83\footnotesize±0.45 & 52.47\footnotesize±0.06 & 62.74\footnotesize±0.23 \\
 & EERM &83.00\footnotesize±0.77&74.76\footnotesize±1.15&OOM& OOM & OOM & OOM & 67.50\footnotesize±0.74 & 51.88\footnotesize±0.07 & 62.56\footnotesize±0.02 \\
 & CaNet &96.12\footnotesize±1.04&\underline{94.57\footnotesize±1.92}&88.82\footnotesize±2.30& \underline{59.01\footnotesize±0.30} & \underline{56.88\footnotesize±0.70} & \underline{56.27\footnotesize±1.21} & \underline{67.47\footnotesize±0.32} & \underline{53.59\footnotesize±0.19} & \underline{64.24\footnotesize±0.18} \\ 
  & CIA-LRA &\underline{97.43\footnotesize±0.28}
&94.22\footnotesize±0.31
&\underline{89.97\footnotesize±0.17}& OOM & OOM & OOM & 67.19\footnotesize±0.43 & 53.08\footnotesize±0.92 & 64.03\footnotesize±0.32
\\
   & \textbf{CGRL)} &\textbf{99.14\footnotesize±0.16}
& \textbf{97.96\footnotesize±0.23}
&\textbf{96.90\footnotesize±0.44}
& \textbf{61.43\footnotesize±0.32
} & \textbf{60.12\footnotesize±0.56
} & \textbf{59.45\footnotesize±0.82
} & \textbf{68.02\footnotesize±0.29} & \textbf{54.98\footnotesize±0.23} & \textbf{64.89\footnotesize±0.12} \\
\midrule
\multirow{11}{*}{GAT} & ERM &91.10\footnotesize±2.26&82.60\footnotesize±0.51&84.80\footnotesize±1.47& 57.15\footnotesize±0.25 & 55.07\footnotesize±0.58 & 46.22\footnotesize±0.82 & 65.67\footnotesize±0.02 & 52.00\footnotesize±0.10 & 61.85\footnotesize±0.05 \\
 & IRM &91.63\footnotesize±1.27&82.73\footnotesize±0.37&84.95\footnotesize±1.06& 56.55\footnotesize±0.18 & 54.53\footnotesize±0.32 & 46.01\footnotesize±0.33 & 67.27\footnotesize±0.19 & 52.85\footnotesize±0.15 & 62.40\footnotesize±0.24 \\
 & Coral &91.82\footnotesize±1.30&82.44\footnotesize±0.58&85.07\footnotesize±0.95& 57.40\footnotesize±0.51 & 55.14\footnotesize±0.71 & 46.71\footnotesize±0.61 & 67.12\footnotesize±0.03 & 52.61\footnotesize±0.01 & 63.41\footnotesize±0.01 \\
 & DANN &92.40\footnotesize±2.05&82.49\footnotesize±0.67&83.94\footnotesize±0.84& 57.23\footnotesize±0.18 & 55.13\footnotesize±0.46 & 46.61\footnotesize±0.57 & 66.59\footnotesize±0.38 & 52.88\footnotesize±0.12 & 62.47\footnotesize±0.32 \\
 & GroupDRO &90.54\footnotesize±0.94& 82.64\footnotesize±0.61&85.17\footnotesize±0.86& 56.69\footnotesize±0.27 & 54.51\footnotesize±0.49 & 46.00\footnotesize±0.59 & 67.41\footnotesize±0.04 & 52.99\footnotesize±0.08 & 62.29\footnotesize±0.03 \\
 & Mixup &92.94\footnotesize±1.21&82.77\footnotesize±0.30&81.58\footnotesize±0.65& 57.17\footnotesize±0.33 & 55.33\footnotesize±0.37 & 47.17\footnotesize±0.84 & 65.58\footnotesize±0.13 & 52.04\footnotesize±0.04 & 61.75\footnotesize±0.13 \\
 & SRGNN &91.77\footnotesize±2.43&82.72\footnotesize±0.35&83.40\footnotesize±0.67& 56.69\footnotesize±0.38 & 55.01\footnotesize±0.55 & 46.88\footnotesize±0.58 & 66.17\footnotesize±0.03 & 52.84\footnotesize±0.04 & 62.07\footnotesize±0.04 \\
 & EERM &91.80\footnotesize±0.73&74.07\footnotesize±0.75&OOM& OOM & OOM & OOM & 66.80\footnotesize±0.46 & 52.39\footnotesize±0.20 & 62.07\footnotesize±0.68 \\
 & CaNet &97.30\footnotesize±0.25&95.33\footnotesize±0.33&89.89\footnotesize±1.92& \underline{60.44\footnotesize±0.27} & \underline{58.54\footnotesize±0.72} & \underline{59.61\footnotesize±0.28} & \underline{68.08\footnotesize±0.19} & \underline{53.49\footnotesize±0.14} & \underline{63.76\footnotesize±0.17} \\  
 & CIA-LRA &\underline{97.89\footnotesize±0.34}
&\underline{95.47\footnotesize±0.09}&\underline{92.12\footnotesize±0.22}
& OOM & OOM & OOM & 67.76\footnotesize±0.29
 & 53.26\footnotesize±0.37 & 63.68\footnotesize±0.25\\

& \textbf{CGRL} &\textbf{98.56\footnotesize±0.17}
& \textbf{97.74\footnotesize±0.16}
&\textbf{96.83\footnotesize±0.21}
& \textbf{61.57\footnotesize±0.32} & \textbf{60.40\footnotesize±0.48} & \textbf{60.78\footnotesize±0.41}& \textbf{68.42\footnotesize±0.25}  & \textbf{54.51\footnotesize±0.17} & \textbf{64.23\footnotesize±0.21} \\
\bottomrule[1pt]
\end{tabular}
}
\label{tab:performance_comparison}

\end{table*}
\subsection{Overall Performance Comparison}

Tables~\ref{table1} and \ref{table2} report the OOD generalization results of CGRL and the competing baselines under feature, spatiotemporal, covariate, and concept shifts. Overall, CGRL exhibits strong consistency across both GCN and GAT backbones. It achieves the best performance in all settings of Table~\ref{table1}, and remains the top-performing method in nearly all GOOD benchmark settings of Table~\ref{table2}, with only one second-best case. These results show that the proposed framework generalizes robustly across diverse graph OOD scenarios rather than benefiting from a particular backbone or shift type.

\subsubsection{Generalization on Feature and Spatiotemporal Shifts}

As shown in Table \ref{table1}, CGRL consistently outperforms all baselines on the three datasets with feature shifts. The gains are particularly notable on Pubmed, where CGRL improves substantially over the strongest competing methods under both backbones. This suggests that the proposed causal-guided disentanglement is especially effective when feature shifts significantly distort the latent geometry of node representations. On Cora and Citeseer, CGRL also maintains clear and stable advantages, further verifying its effectiveness on small- and medium-scale citation graphs.

The advantage of CGRL further extends to spatiotemporal shifts. On Arxiv, it achieves the best performance on all temporal splits with both GCN and GAT, and on Twitch it again ranks first across all target domains. Notably, several strong baselines, such as EERM and CIA-LRA, encounter out-of-memory failures on Pubmed and Arxiv, whereas CGRL remains trainable and still achieves the best results. This improvement is not only due to better representation learning, but also to the scalable optimization introduced in Section \ref{section4.4.4}. In particular, for the large-scale datasets Pubmed and Arxiv, we use the prototype-and-sampling implementation of the pairwise regularizers, which preserves the optimization role of $\mathcal{L}^{sc}_{intra}$ and $\mathcal{L}^{sc}_{inter}$ while substantially reducing their computational cost. As a result, CGRL attains a favorable balance between robustness and scalability on large graphs.

\subsubsection{Generalization on Covariate and Concept Shifts}

Table \ref{table2} further evaluates the methods on the GOOD benchmark under controlled covariate and concept shifts. Under covariate shifts, CGRL achieves the best results on all datasets and under both backbones, with particularly clear improvements on GOODCora and GOODWebKB. This indicates that the proposed framework effectively alleviates the mismatch between training and test input distributions by promoting stronger intra-class compactness and weaker inter-class contamination. Under concept shifts, CGRL remains highly competitive and achieves the best performance in almost all settings, except for GOODCBAS with the GCN backbone, where it ranks second. Even in this case, its performance remains close to the best result. Overall, these results show that CGRL is robust not only to input-level shifts, but also to changes in the input--label relationship, confirming its effectiveness across multiple OOD regimes.

\begin{table*}[htbp]

    \centering
    \caption{Test (mean±standard deviation) Accuracy (\%) on OOD data of GOODCora, GOODCBAS and GOOODWebKB datasets under the different shifts, respectively. OOM denotes out-of-memory. \textbf{Bold} numbers indicate the best result, and the \underline{underlined} numbers indicate the second-best result.}
    \label{table2}
\resizebox{\textwidth}{!}{
\begin{tabular}{cccccccccc}

\toprule[1pt]
\multirow{3}{*}{\textbf{Backbone}} & \multirow{3}{*}{\textbf{Method}} & \multicolumn{4}{c}{\textbf{Covariate}} & \multicolumn{4}{c}{\textbf{Concept}}\\
\cmidrule(lr){3-6} \cmidrule(lr){7-10}
&& \multicolumn{2}{c}{\textbf{GOODCora}} & \textbf{GOODCBAS}&\textbf{GOODWebKB} & \multicolumn{2}{c}{\textbf{GOODCora}} & \textbf{GOODCBAS}&\textbf{GOODWebKB}\\
\cmidrule(lr){3-4} \cmidrule(lr){7-8}
&& \textbf{degree}&\textbf{word}&\textbf{color}&\textbf{university}& \textbf{degree}&\textbf{word}&\textbf{color}&\textbf{university}\\
\midrule
\multirow{11}{*}{GCN} & ERM & 55.78\footnotesize±0.52 & 64.76\footnotesize±0.30 & 78.57\footnotesize±2.02 & 16.14\footnotesize±1.35 & 60.24
\footnotesize±0.40 & 64.32
\footnotesize±0.15&82.14\footnotesize±1.17&27.52\footnotesize±0.75 \\
 & IRM & 55.77\footnotesize±0.66 & 64.81\footnotesize±0.33 & 78.57\footnotesize±1.17 & 13.75\footnotesize±4.91 & 61.23\footnotesize±0.08 & 64.42\footnotesize±0.18&81.67\footnotesize±0.89&27.52\footnotesize±0.75 \\
 & Coral & 56.03\footnotesize±0.37 & 64.75\footnotesize±0.26 & 78.09\footnotesize±0.67 & 11.90\footnotesize±1.72 & 60.41\footnotesize±0.27 & 64.34\footnotesize±0.17&82.86\footnotesize±0.58&26.61\footnotesize±0.75 \\
 & DANN & 56.10\footnotesize±0.59 & 64.77\footnotesize±0.42 & 77.57
\footnotesize±2.86 & 15.08\footnotesize±0.37 & 60.78\footnotesize±0.38 & 64.51
\footnotesize±0.19&82.50\footnotesize±0.72&26.91\footnotesize±0.36 \\
 & GroupDRO & 55.64\footnotesize±0.50 & 64.62
\footnotesize±0.30 & 79.52\footnotesize±0.67 & 14.29\footnotesize±2.59 & 60.59
\footnotesize±0.36 & 64.34\footnotesize±0.25&82.38\footnotesize±0.67&28.44\footnotesize±0.01 \\
 & Mixup & 57.89\footnotesize±0.27 & 65.07\footnotesize±0.22&70.00\footnotesize±5.34& 16.67\footnotesize±1.12 & 63.65\footnotesize±0.37 & 64.45\footnotesize±0.12 & 65.48\footnotesize±0.67&30.28\footnotesize±1.50
 \\
 & SRGNN & 57.13\footnotesize±0.25 & 64.50\footnotesize±0.35 & 73.81\footnotesize±4.71 & 16.40\footnotesize±1.63 & 61.21
\footnotesize±0.29 & 64.53\footnotesize±0.27&80.95\footnotesize±0.67&27.52\footnotesize±0.75
 \\
 & EERM & 56.88\footnotesize±0.32& 61.98\footnotesize±0.1& 40.48\footnotesize±9.87
 & 16.21\footnotesize±5.67 & 58.38\footnotesize±0.04 & 63.09\footnotesize±0.36&61.43\footnotesize±1.17& 28.04\footnotesize±11.67\\
 & CaNet & 57.35\footnotesize±0.04 & 64.66\footnotesize±0.36 & 80.95
\footnotesize±0.67 & 15.61\footnotesize±5.51 & 60.34
\footnotesize±0.20 &64.65\footnotesize±0.39&83.24\footnotesize±3.32
&26.30\footnotesize±0.43\\ 
  & CIA-LAR & \underline{58.40\footnotesize±0.59} & \underline{65.95\footnotesize±0.04} & \underline{82.86\footnotesize±1.17} & \underline{19.84\footnotesize±2.83} & \underline{63.71
\footnotesize±0.32} &\underline{65.07\footnotesize±0.21}&\textbf{94.53\footnotesize±0.33}
&\underline{36.70\footnotesize±0.75}\\
   & \textbf{CGRL} & \textbf{59.68\footnotesize±0.47} & \textbf{69.01\footnotesize±0.21} & \textbf{83.14\footnotesize±1.51} & \textbf{23.02\footnotesize±2.37} & \textbf{65.51\footnotesize±0.23} & \textbf{68.15\footnotesize±0.28}&\underline{85.29\footnotesize±1.13}&\textbf{39.27\footnotesize±1.02}\\
\midrule
\multirow{11}{*}{GAT} & ERM & 55.30\footnotesize±0.34 & 63.75\footnotesize±0.39 & 65.24\footnotesize±2.69 & 31.48\footnotesize±6.12 & 61.36\footnotesize±0.38& 64.38\footnotesize±0.13&74.52\footnotesize±1.87&25.69\footnotesize±1.98
 \\
 & IRM & 55.07\footnotesize±0.30 & 63.75\footnotesize±0.26 & 65.24
\footnotesize±1.78 & 30.69\footnotesize±8.63&61.42\footnotesize±0.34 & 64.31\footnotesize±0.39 & 73.33\footnotesize±0.89&25.69\footnotesize±1.98
 \\
 & Coral & 55.35\footnotesize±0.40 & 63.96\footnotesize±0.07 & 64.76
\footnotesize±2.43 & 33.86\footnotesize±6.93 & 61.62\footnotesize±0.26 & 64.26\footnotesize±0.28&75.95\footnotesize±3.21& 30.27\footnotesize±3.26
\\
 & DANN & 55.31\footnotesize±0.57 & 64.21\footnotesize±0.12 & 65.44\footnotesize±3.18 & 32.43\footnotesize±2.09 & 61.59\footnotesize±0.31& 64.41\footnotesize±±0.24&74.04\footnotesize±2.38&26.87\footnotesize±2.07
 \\
 & GroupDRO & 55.03\footnotesize±0.45 & 63.82\footnotesize±0.06 & 67.62\footnotesize±2.43 & 31.75\footnotesize±2.82 &61.31\footnotesize±0.20 & 64.07\footnotesize±0.25&73.81\footnotesize±1.78&26.91\footnotesize±2.16
 \\
 & Mixup & 56.77\footnotesize±0.36 & 65.70\footnotesize±0.28 &63.33
\footnotesize±8.60 & 20.37\footnotesize±11.38 & 63.97\footnotesize±0.18& 65.42\footnotesize±0.32&73.33\footnotesize±1.47&\underline{38.53\footnotesize±0.75}
 \\
 & SRGNN & 55.87\footnotesize±0.32 & 64.50\footnotesize±0.35 & 68.09\footnotesize±0.67 & 28.84\footnotesize±1.35 & 61.21\footnotesize±0.29 & 64.10\footnotesize±±0.28&72.38\footnotesize±1.22&23.55\footnotesize±1.56
 \\
 & EERM & 46.63\footnotesize±1.75& 62.57\footnotesize±0.50& 60.47\footnotesize±4.10& 33.33\footnotesize±14.60 & 48.05\footnotesize±2.03 & 53.02\footnotesize±1.23&60.95\footnotesize±3.56& 25.38\footnotesize±4.26
\\
 & CaNet & 55.35\footnotesize±0.14 & 62.76\footnotesize±0.25 & 68.09\footnotesize±1.78 & 23.87\footnotesize±15.16 & 60.97\footnotesize±0.07& 63.73\footnotesize±0.44&75.95\footnotesize±3.41& 24.77\footnotesize±3.97
\\  
 & CIA-LRA & \underline{57.95\footnotesize±0.13} & \underline{68.59\footnotesize±0.26} & \underline{75.24\footnotesize±1.78} & \underline{38.62\footnotesize±3.57} & \underline{67.08\footnotesize±0.26} & \underline{68.05\footnotesize±0.14}&\underline{78.34\footnotesize±3.51}&31.80\footnotesize±1.88
\\

& \textbf{CGRL} & \textbf{60.05\footnotesize±0.14} & \textbf{69.37\footnotesize±0.17} & \textbf{77.56\footnotesize±2.59} & \textbf{42.33\footnotesize±3.28} & \textbf{68.67\footnotesize±0.43} & \textbf{70.84\footnotesize±0.45} &\textbf{83.57\footnotesize±1.01}& \textbf{39.45\footnotesize±1.54}
\\
\bottomrule[1pt]
\end{tabular}
}
\label{tab:performance_comparison}

\end{table*}
\subsection{Alleviating the Info-Jitter Phenomenon (RQ2)}

To evaluate whether CGRL can alleviate the \emph{Info-Jitter} phenomenon, we monitor the mutual information (MI) between the predictive representations of all nodes and their ground-truth labels throughout training. Following the protocol in Section~5.1, we further summarize the MI dynamics using two complementary metrics: \emph{Jitter}, which measures the cumulative fluctuation of the MI trajectory, and \emph{TailMI}, which measures the average MI over epochs 100--500. All quantitative results in Table~\ref{tab:mi_fluctuation} are reported as mean$\pm$standard deviation over five random seeds. Since a low Jitter alone may also arise from uniformly weak representations, Jitter should be interpreted jointly with TailMI: a desirable model should simultaneously exhibit reduced fluctuation and retain label-relevant information. 

\begin{figure}[t]
\begin{center}
\centerline{\includegraphics[width=1\columnwidth]{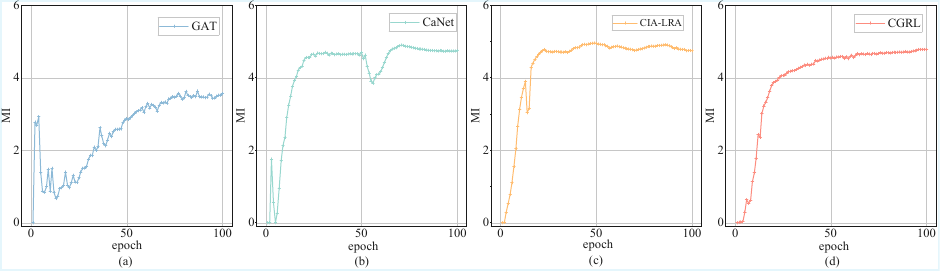}}
\caption{Mutual information between predictive representations and ground-truth labels on the GOODCora dataset with degree-domain covariate shifts. Subfigures (a)–(d) correspond to GAT-based ERM, CANet, CIA-LRA, and CGRL, respectively.}
\label{figure4}
\end{center}
\end{figure}
\begin{figure}[t]
	\begin{center}
		\centerline{\includegraphics[width=1\columnwidth]{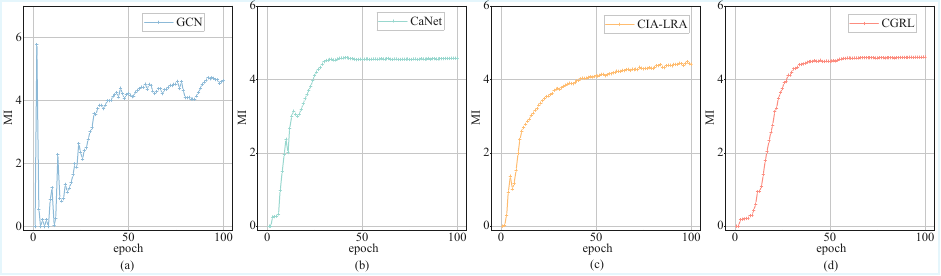}}
		\caption{Mutual information between predictive representations and ground-truth labels on the GOODCora dataset with degree-domain concept shifts. Subfigures (a)–(d) correspond to GCN-based ERM, CANet, CIA-LRA, and CGRL, respectively.}
		\label{figure5}
	\end{center}
\end{figure}
\begin{figure}[h]
	\begin{center}
		\centerline{\includegraphics[width=1\columnwidth]{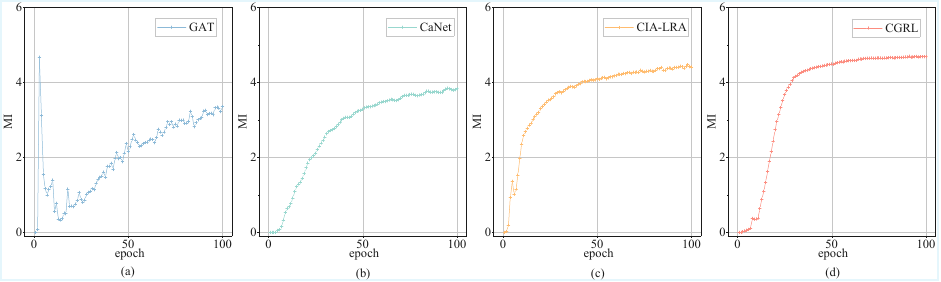}}
		\caption{Mutual information between predictive representations and ground-truth labels on the GOODCora dataset with word-domain concept shifts. Subfigures (a)–(d) correspond to GAT-based ERM, CANet, CIA-LRA, and CGRL, respectively.}
		\label{figure6}
	\end{center}

\end{figure}

\begin{table}[t]
	\centering

	\caption{Comparison of mutual information fluctuation metrics on datasets under the sptiotemporal shifts and different backbones.}
	\label{tab:mi_fluctuation}
	\resizebox{\textwidth}{!}{
		\begin{tabular}{cccccccccccc}
			\toprule
			\multirow{2}{*}{Backbone} & \multirow{2}{*}{Method}
			& \multicolumn{2}{c}{Cora}
			& \multicolumn{2}{c}{Citeseer}
			& \multicolumn{2}{c}{Pubmed}
			& \multicolumn{2}{c}{Arxiv}
			& \multicolumn{2}{c}{Twitch} \\
			\cmidrule(lr){3-4}
			\cmidrule(lr){5-6}
			\cmidrule(lr){7-8}
			\cmidrule(lr){9-10}
			\cmidrule(lr){11-12}
			& & Jitter & TailMI & Jitter & TailMI & Jitter & TailMI & Jitter & TailMI & Jitter & TailMI \\
			\midrule
			\multirow{11}{*}{GCN}
			& ERM      & 2.21\footnotesize$\pm$0.08 & 2.47{\footnotesize$\pm$0.02} & 2.27{\footnotesize$\pm$0.21} & 2.33{\footnotesize$\pm$0.03} & 5.07{\footnotesize$\pm$0.29} & 1.17{\footnotesize$\pm$0.01} & 10.16{\footnotesize$\pm$0.86} & 1.71{\footnotesize$\pm$0.04} & 6.35{\footnotesize$\pm$0.53} & 0.31{\footnotesize$\pm$0.02} \\
			& IRM      & 4.17{\footnotesize$\pm$1.26} & 2.32{\footnotesize$\pm$0.01} & 7.22{\footnotesize$\pm$1.75} & 2.11{\footnotesize$\pm$0.03} & 9.44{\footnotesize$\pm$0.43} & 1.13{\footnotesize$\pm$0.01} & 13.39{\footnotesize$\pm$3.01} & 1.85{\footnotesize$\pm$0.01} & 5.73{\footnotesize$\pm$0.77} & 0.09{\footnotesize$\pm$0.00} \\
			& Coral    & 1.77{\footnotesize$\pm$0.21} & 1.70{\footnotesize$\pm$0.59} & 1.60{\footnotesize$\pm$0.28} & 1.57{\footnotesize$\pm$0.53} & 2.57{\footnotesize$\pm$0.28} & 0.85{\footnotesize$\pm$0.19} & 4.39{\footnotesize$\pm$2.08} & 0.92{\footnotesize$\pm$0.35} & 0.83{\footnotesize$\pm$1.02} & 0.07{\footnotesize$\pm$0.08} \\
			& DANN     & 1.87{\footnotesize$\pm$0.27} & 1.71{\footnotesize$\pm$0.62} & 1.74{\footnotesize$\pm$0.24} & 1.69{\footnotesize$\pm$0.52} & 2.47{\footnotesize$\pm$0.15} & 1.09{\footnotesize$\pm$0.02} & 4.37{\footnotesize$\pm$3.35} & 0.99{\footnotesize$\pm$0.44} & 0.66{\footnotesize$\pm$1.15} & 0.07{\footnotesize$\pm$0.04} \\
			& GroupDRO & 1.85{\footnotesize$\pm$0.23} & 1.69{\footnotesize$\pm$0.58} & 1.77{\footnotesize$\pm$0.21} & 1.51{\footnotesize$\pm$0.49} & 2.87{\footnotesize$\pm$0.12} & 0.81{\footnotesize$\pm$0.18} & 5.53{\footnotesize$\pm$1.37} & 1.08{\footnotesize$\pm$0.45} & 0.70{\footnotesize$\pm$0.07} & 0.06{\footnotesize$\pm$0.03} \\
			& Mixup    & 0.85{\footnotesize$\pm$0.61} & 1.76{\footnotesize$\pm$0.51} & 1.20{\footnotesize$\pm$0.54} & 1.66{\footnotesize$\pm$0.41} & 2.36{\footnotesize$\pm$0.03} & 0.94{\footnotesize$\pm$0.09} & 3.07{\footnotesize$\pm$0.90} & 1.41{\footnotesize$\pm$0.42} & 0.81{\footnotesize$\pm$0.03} & 0.05{\footnotesize$\pm$0.02} \\
			& SRGNN    & 1.07{\footnotesize$\pm$0.06} & 2.47{\footnotesize$\pm$0.01} & 0.90{\footnotesize$\pm$0.09} & 2.36{\footnotesize$\pm$0.08} & 2.66{\footnotesize$\pm$0.10} & 1.01{\footnotesize$\pm$0.02} & 4.28{\footnotesize$\pm$0.10} & 0.98{\footnotesize$\pm$0.04} & 4.34{\footnotesize$\pm$0.08} & 0.16{\footnotesize$\pm$0.00} \\
			& EERM     & 2.43{\footnotesize$\pm$0.90} & 1.17{\footnotesize$\pm$0.04} & 1.00{\footnotesize$\pm$0.10} & 0.47{\footnotesize$\pm$0.01} & -- & -- & -- & -- & 5.21{\footnotesize$\pm$1.23} & 0.18{\footnotesize$\pm$0.02} \\
			& CaNet    & 0.43{\footnotesize$\pm$0.02} & 2.54{\footnotesize$\pm$0.00} & 0.53{\footnotesize$\pm$0.02} & 2.43{\footnotesize$\pm$0.00} & 1.69{\footnotesize$\pm$0.17} & 1.20{\footnotesize$\pm$0.01} & 5.54{\footnotesize$\pm$0.77} & 1.05{\footnotesize$\pm$0.04} & 4.58{\footnotesize$\pm$0.55} & 0.22{\footnotesize$\pm$0.00} \\
			& CIA-LRA  & 1.20{\footnotesize$\pm$0.52} & 1.22{\footnotesize$\pm$0.02} & 0.66{\footnotesize$\pm$0.05} & 2.34{\footnotesize$\pm$0.02} & -- & -- & -- & -- & 4.89{\footnotesize$\pm$0.41} & 0.26{\footnotesize$\pm$0.01} \\
			& CGRL     & 0.40{\footnotesize$\pm$0.02} & 2.62{\footnotesize$\pm$0.03} & 0.44{\footnotesize$\pm$0.08} & 2.51{\footnotesize$\pm$0.04} & 1.40{\footnotesize$\pm$0.08} & 1.26{\footnotesize$\pm$0.01} & 5.09{\footnotesize$\pm$0.43} & 1.94{\footnotesize$\pm$0.12} & 5.77{\footnotesize$\pm$0.83} & 0.36{\footnotesize$\pm$0.02} \\
			\midrule
			\multirow{11}{*}{GAT}
			& ERM      & 2.82{\footnotesize$\pm$0.56} & 1.98{\footnotesize$\pm$0.02} & 2.57{\footnotesize$\pm$0.52} & 2.36{\footnotesize$\pm$0.01} & 10.7165{\footnotesize$\pm$0.32} & 1.31{\footnotesize$\pm$0.12} & 9.17{\footnotesize$\pm$1.45} & 1.54{\footnotesize$\pm$0.05} & 9.00{\footnotesize$\pm$0.84} & 0.26{\footnotesize$\pm$0.05} \\
			& IRM      & 2.95{\footnotesize$\pm$0.90} & 2.32{\footnotesize$\pm$0.06} & 4.92{\footnotesize$\pm$0.98} & 2.18{\footnotesize$\pm$0.05} & 4.07{\footnotesize$\pm$0.44} & 1.13{\footnotesize$\pm$0.01} & 13.63{\footnotesize$\pm$2.33} & 1.49{\footnotesize$\pm$0.01} & 4.82{\footnotesize$\pm$0.41} & 0.09{\footnotesize$\pm$0.00} \\
			& Coral    & 2.11{\footnotesize$\pm$0.37} & 1.87{\footnotesize$\pm$0.58} & 1.99{\footnotesize$\pm$0.33} & 1.66{\footnotesize$\pm$0.55} & 3.80{\footnotesize$\pm$0.26} & 0.80{\footnotesize$\pm$0.18} & 4.51{\footnotesize$\pm$0.14} & 1.13{\footnotesize$\pm$0.47} & 1.09{\footnotesize$\pm$0.05} & 0.06{\footnotesize$\pm$0.03} \\
			& DANN     & 1.77{\footnotesize$\pm$0.37} & 1.88{\footnotesize$\pm$0.58} & 1.62{\footnotesize$\pm$0.34} & 1.66{\footnotesize$\pm$0.55} & 3.73{\footnotesize$\pm$0.26} & 0.88{\footnotesize$\pm$0.18} & 4.51{\footnotesize$\pm$0.14} & 1.12{\footnotesize$\pm$0.47} & 1.11{\footnotesize$\pm$0.06} & 0.08{\footnotesize$\pm$0.03} \\
			& GroupDRO & 1.68{\footnotesize$\pm$0.34} & 1.84{\footnotesize$\pm$0.56} & 1.61{\footnotesize$\pm$0.28} & 1.59{\footnotesize$\pm$0.51} & 4.40{\footnotesize$\pm$0.25} & 1.03{\footnotesize$\pm$0.17} & 4.61{\footnotesize$\pm$0.11} & 1.12{\footnotesize$\pm$0.46} & 1.14{\footnotesize$\pm$0.06} & 0.08{\footnotesize$\pm$0.04} \\
			& Mixup    & 0.89{\footnotesize$\pm$0.61} & 1.94{\footnotesize$\pm$0.52} & 0.86{\footnotesize$\pm$0.55} & 1.80{\footnotesize$\pm$0.44} & 3.75{\footnotesize$\pm$0.08} & 0.88{\footnotesize$\pm$0.10} & 4.97{\footnotesize$\pm$0.28} & 1.65{\footnotesize$\pm$0.47} & 1.02{\footnotesize$\pm$0.04} & 0.07{\footnotesize$\pm$0.02} \\
			& SRGNN    & 0.96{\footnotesize$\pm$0.04} & 2.51{\footnotesize$\pm$0.00} & 0.93{\footnotesize$\pm$0.25} & 2.39{\footnotesize$\pm$0.02} & 4.21{\footnotesize$\pm$0.08} & 0.91{\footnotesize$\pm$0.03} & 5.16{\footnotesize$\pm$0.18} & 1.37{\footnotesize$\pm$0.26} & 4.82{\footnotesize$\pm$0.32} & 0.20{\footnotesize$\pm$0.00} \\
			& EERM     & 2.83{\footnotesize$\pm$0.82} & 1.14{\footnotesize$\pm$0.03} & 1.55{\footnotesize$\pm$0.12} & 0.55{\footnotesize$\pm$0.02} & -- & -- & -- & -- & 4.38{\footnotesize$\pm$0.34} & 0.18{\footnotesize$\pm$0.01} \\
			& CaNet    & 0.80{\footnotesize$\pm$0.07} & 2.48{\footnotesize$\pm$0.00} & 0.66{\footnotesize$\pm$0.11} & 2.38{\footnotesize$\pm$0.01} & 2.31{\footnotesize$\pm$0.46} & 1.16{\footnotesize$\pm$0.01} & 4.56{\footnotesize$\pm$1.22} & 1.58{\footnotesize$\pm$0.17} & 5.98{\footnotesize$\pm$0.34} & 0.19{\footnotesize$\pm$0.01} \\
			& CIA-LRA  & 1.33{\footnotesize$\pm$0.23} & 2.34{\footnotesize$\pm$0.08} & 1.21{\footnotesize$\pm$0.12} & 2.28{\footnotesize$\pm$0.00} & -- & -- & -- & -- & 6.22{\footnotesize$\pm$0.31} & 0.21{\footnotesize$\pm$0.03} \\
			& CGRL     & 0.61{\footnotesize$\pm$0.09} & 2.58{\footnotesize$\pm$0.01} & 0.47{\footnotesize$\pm$0.05} & 2.47{\footnotesize$\pm$0.01} & 3.17{\footnotesize$\pm$0.32} & 1.278{\footnotesize$\pm$0.03} & 6.56{\footnotesize$\pm$0.87} & 1.65{\footnotesize$\pm$0.01} & 6.11{\footnotesize$\pm$0.32} & 0.31{\footnotesize$\pm$0.02} \\
			\bottomrule
		\end{tabular}
	}
\end{table}

Figures \ref{figure4}--\ref{figure6} provide representative visualizations on GOODCora under controlled distribution shifts. Specifically, Figure \ref{figure4} considers degree-domain covariate shifts with the GAT backbone, Figure \ref{figure5} considers degree-domain concept shifts with the GCN backbone, and Figure \ref{figure6} considers word-domain concept shifts with the GAT backbone. Across these settings, vanilla ERM exhibits pronounced MI oscillations, indicating unstable representation--label alignment under OOD shifts. CaNet and CIA-LRA partially reduce such fluctuations, but their trajectories still display noticeable instability or early saturation. In contrast, CGRL yields substantially smoother MI evolution and more stable convergence while maintaining a consistently higher MI level in the later training stage. These observations suggest that CGRL does not merely damp fluctuations, but stabilizes training while preserving stronger class-relevant information in the predictive representation. 

The quantitative results in Table~\ref{tab:mi_fluctuation} further support this conclusion. On Cora and Citeseer, CGRL achieves the lowest or near-lowest Jitter together with the highest TailMI under both GCN and GAT, showing that its improved stability is accompanied by stronger information retention rather than by trivial suppression of MI. On Pubmed, Arxiv, and Twitch, CGRL does not always yield the smallest Jitter; nevertheless, it consistently attains markedly larger TailMI than most competing methods, especially relative to the domain-generalization baselines whose TailMI values remain very small. This distinction is important: although some baselines appear less fluctuating, their low TailMI indicates that they preserve only limited label-relevant information in the predictive representation. By contrast, CGRL maintains substantially stronger semantic alignment between representations and labels, even when the resulting MI trajectory remains moderately dynamic. Therefore, on these more challenging or larger-scale datasets, the slightly higher Jitter of CGRL should be understood together with its much stronger TailMI, which indicates a more informative and still competitive optimization trajectory rather than a collapse of stable learning. 

Overall, the evidence from both the MI trajectories and the summary metrics shows that CGRL alleviates \emph{Info-Jitter} in a substantive manner. It not only stabilizes the representation--label alignment under distribution shifts, but also preserves richer discriminative information in the predictive representations. This behavior is consistent with the design of our framework: the causal modulation mechanism suppresses environment-sensitive interference during message passing, while the intra-class aggregation and inter-class separation objectives further regularize the latent space toward a more stable node-level causal representation.

\subsection{Softmax Confidence}
In this section, we explain why we adopt the Gumbel-softmax relaxation during training instead of directly using a deterministic softmax. The motivation is twofold. First, the Gumbel-softmax provides a differentiable stochastic approximation to discrete selection, which is more suitable for our energy-based reconstruction pipeline and allows stable end-to-end optimization via backpropagation. Second, the injected stochasticity regularizes the causal modulation process and prevents the model from becoming overly confident in in-distribution (ID) patterns. This is particularly important for OOD generalization. When plain softmax becomes highly saturated, the predicted probability of one branch may approach $1$ while the others collapse toward $0$, leading to sharply peaked confidence scores and diminished gradient flow. Such overconfident behavior may fit ID data well, but it also encourages the model to over-specialize to environment-specific patterns, thereby weakening generalization under distribution shifts. By contrast, the Gumbel-softmax maintains a more moderate confidence distribution during training, enabling the model to preserve sufficient uncertainty and learn more robust representations.

To empirically illustrate this effect, Figure \ref{figure7} presents the confidence-score histograms produced by Gumbel-softmax and plain softmax on the Cora and GOODCora datasets with degree-domain covariate shifts, where the y-axis denotes the density at each confidence level. The figure shows that softmax tends to produce heavily concentrated high-confidence outputs, indicating a clear saturation effect. In contrast, the Gumbel-softmax yields a more balanced confidence distribution: it avoids the extreme over-saturation of softmax while not collapsing into uniformly low-confidence predictions. This suggests that the Gumbel-softmax regularizes the training dynamics by preventing blindly overconfident assignments and preserving a more appropriate confidence range. As a result, the model can better resist spurious ID-specific patterns and achieve stronger robustness under distribution shifts.
\begin{figure}[t]
\begin{center}
\centerline{\includegraphics[width=0.8\columnwidth]{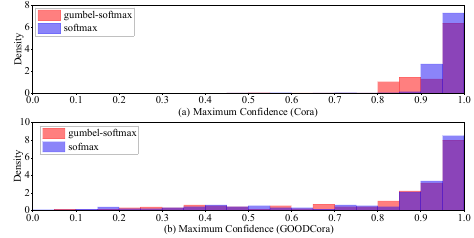}}
\caption{Softmax confidence for Cora and GOODCora datasets.}

\label{figure7}
\end{center}
\end{figure}

\subsection{The Effect of Spurious Correlations on MI Dynamics}

To better understand the origin of the \emph{Info-Jitter} phenomenon, we further compare the MI trajectories on ID and OOD data, as shown in Figure \ref{figure9}. The motivation of this experiment is to examine whether the discrepancy between ID and OOD MI dynamics can serve as empirical evidence of the model's reliance on spurious correlations. Intuitively, if a model overfits environment-specific statistical patterns in ID data, then once these patterns are disrupted under OOD shifts, its representation--label alignment should become less stable, resulting in lower MI values and more pronounced fluctuations. In contrast, if the learned representations are primarily supported by cross-distribution causal mechanisms, the MI trajectories on ID and OOD data should remain more consistent. We summarize the observations as follows. 
First, both GCN and GAT exhibit similar qualitative trends on ID and OOD data, but with a clear degradation under OOD shifts. In the early stage of training, MI increases rapidly, indicating that the models begin to capture predictive correlations from the training graph. However, as training proceeds, the MI trajectories of both models continue to fluctuate without stable convergence. More importantly, the OOD MI is consistently lower than the corresponding ID MI, and its fluctuation is substantially stronger. This pattern suggests that although the two models learn some label-relevant information from ID data, they also rely on unstable spurious correlations. Once these correlations are weakened or broken under distribution shifts, the models cannot maintain stable representation--label alignment, leading to both reduced MI and amplified instability.

Second, CGRL shows a markedly different behavior. Although the MI rises sharply in the early stage, it quickly enters a stable regime and then converges smoothly on both ID and OOD data. Compared with the baselines, the gap between the ID and OOD MI trajectories is much smaller, and the OOD curve remains substantially more stable throughout training. This observation suggests that CGRL depends less on environment-specific spurious patterns and instead learns more robust causal representations that remain effective across distributions. Therefore, Fig \ref{figure9} provides empirical evidence that the disruption of spurious correlations is closely related to MI fluctuation, and that the proposed causal-guided disentanglement strategy effectively mitigates this effect.
\begin{figure}[t]
\begin{center}
\centerline{\includegraphics[width=1\columnwidth]{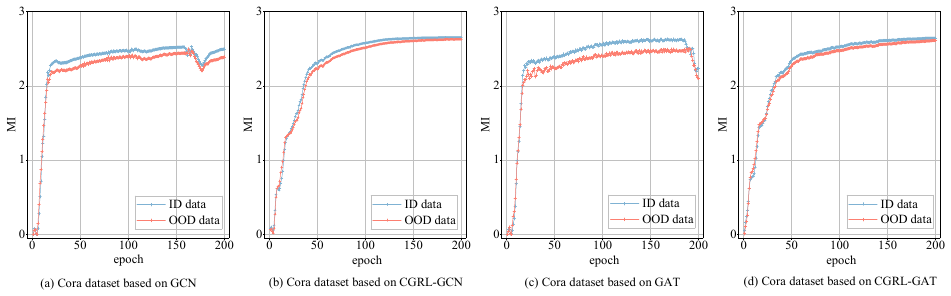}}
\caption{Mutual information on ID and OOD data.}

\label{figure9}
\end{center}
\end{figure}

\begin{table}[t]
	\caption{Ablation results under feature and spatiotemporal shifts across datasets.}
	\label{table7}
	\centering
				\resizebox{\textwidth}{!}{
						\begin{tabular}{cccccccccc}
							\toprule
							\multirow{2}{*}{Methods} & \multirow{2}{*}{Cora} & \multirow{2}{*}{Citeseer} & \multirow{2}{*}{Pubmed}&\multicolumn{3}{c}{Arxiv} & \multicolumn{3}{c}{Twitch}\\
							\cmidrule(lr){5-7} \cmidrule(lr){8-10}\\
							&&&& 2014-2015 & 2016-2017 & 2018-2020 & ES & FR & EN\\
							\midrule
							w/o $\mathcal{L}_{sup}$  & 70.34\footnotesize±1.17& 69.89\footnotesize±2.31& 78.52\footnotesize±2.17&46.23\footnotesize±0.27 & 45.14\footnotesize±0.28 & 42.17\footnotesize±0.36 & 65.47\footnotesize±0.08& 49.48\footnotesize±0.12& 58.78\footnotesize±0.04\\
							w/o $\mathcal{L}_{rec}$ & 91.78\footnotesize±0.32& 87.12\footnotesize±0.19&85.73\footnotesize±0.41 & 49.82\footnotesize±0.18 & 48.43\footnotesize±0.21 & 46.77\footnotesize±0.67 & 66.56\footnotesize±0.27& 52.67\footnotesize±0.19& 62.76\footnotesize±0.08\\
							w/o $\mathcal{L}_{intra}$/$\mathcal{L}^{sc}_{intra}$   & 96.64\footnotesize±0.23& 94.14\footnotesize±0.21& 90.81\footnotesize±0.59& 58.63\footnotesize±0.44 & 58.64\footnotesize±0.63 & 56.12\footnotesize±1.02 & 66.89\footnotesize±0.18& 52.72\footnotesize±0.34& 62.35\footnotesize±0.13\\
							w/o $\mathcal{L}_{inter}$/$\mathcal{L}^{sc}_{inter}$&97.57\footnotesize±0.21&95.40\footnotesize±0.18&92.47\footnotesize±0.31&58.83\footnotesize±0.24 & 58.22\footnotesize±0.28 & 57.37\footnotesize±0.98 & 67.05\footnotesize±0.32& 53.01\footnotesize±0.29& 63.10\footnotesize±0.15\\
							
							CGRL&99.14\footnotesize±0.16&97.96\footnotesize±0.23&96.90\footnotesize±0.44&61.57\footnotesize±0.32 & 60.12\footnotesize±0.56 & 59.45\footnotesize±0.82 & 68.02\footnotesize±0.29& 54.98\footnotesize±0.23& 64.89\footnotesize±0.12 \\
							\bottomrule
						\end{tabular}
					}
\end{table}

\subsection{Ablation Studies}

We conduct ablation studies to examine the contribution of each component in CGRL. Specifically, we consider four variants: w/o $\mathcal{L}_{sup}$, w/o $\mathcal{L}_{rec}$, w/o $\mathcal{L}_{intra}$/$\mathcal{L}^{sc}_{intra}$, and w/o $\mathcal{L}_{inter}$/$\mathcal{L}^{sc}_{inter}$, where ``w/o'' denotes removing the corresponding objective from the full model. Table~\ref{table7} reports the results under feature and spatiotemporal shifts, while Fig \ref{figure9} further presents the corresponding ablation trends under covariate and concept shifts.

Specifically, a consistent pattern can be observed across all shift types: removing any component leads to a clear performance drop, which shows that the full effectiveness of CGRL relies on the joint action of supervised prediction, structural reconstruction, and latent disentanglement. Among all variants, w/o $\mathcal{L}_{sup}$ suffers the most severe degradation, indicating that the supervision term is indispensable for preserving discriminative predictive representations. Without direct label supervision, the remaining objectives alone cannot maintain satisfactory node classification performance. The reconstruction term $\mathcal{L}_{rec}$ also plays a crucial role. As shown in Table~\ref{table7}, removing $\mathcal{L}_{rec}$ causes substantial degradation, especially on Pubmed, Arxiv, and Twitch, suggesting that energy-based reconstruction is important for stabilizing structural representation learning under more challenging distribution shifts. In addition, both $\mathcal{L}_{intra}$/$\mathcal{L}^{sc}_{intra}$ and $\mathcal{L}_{inter}$/$\mathcal{L}^{sc}_{inter}$ contribute consistently to the final performance. Removing either one weakens the model across datasets, confirming that intra-class compactness and inter-class separation are both necessary for node-level causal disentanglement.

Fig \ref{figure9} shows that the same conclusion remains valid under covariate and concept shifts: the full CGRL model consistently outperforms all ablated variants, and no single component can recover the performance of the complete framework on its own. Taken together, these results verify that CGRL benefits from the coordinated interaction of all its objectives, rather than from any isolated loss term.

\begin{figure}[h]
	\begin{center}
		\centerline{\includegraphics[width=1.0\columnwidth]{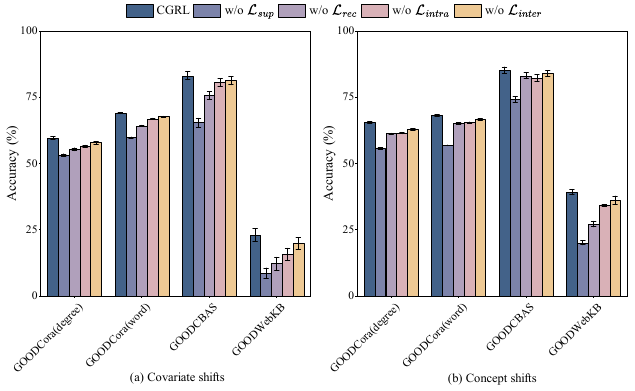}}
		\caption{Ablation results under covariate and concept shifts across datasets.}
		
		\label{figure9}
	\end{center}
\end{figure}

\begin{figure}[h]
	\begin{center}
		\centerline{\includegraphics[width=1.0\columnwidth]{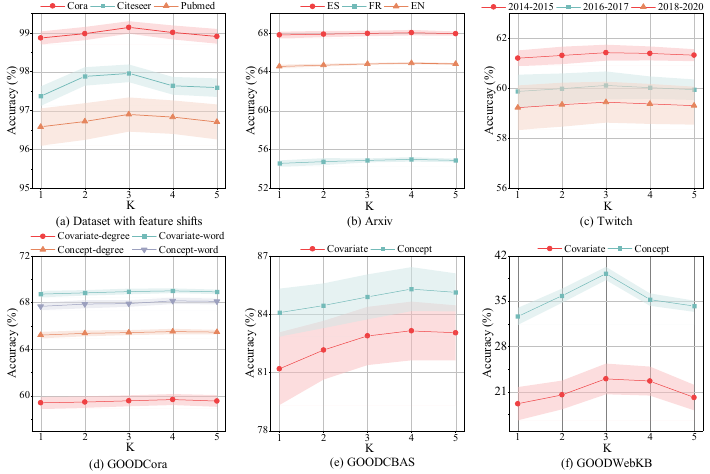}}
		\caption{The sensitivity of branches K to all datasets.}
		
		\label{figure10}
	\end{center}
\end{figure}


\begin{figure}[t]
	\begin{center}
		\centerline{\includegraphics[width=1.0\columnwidth]{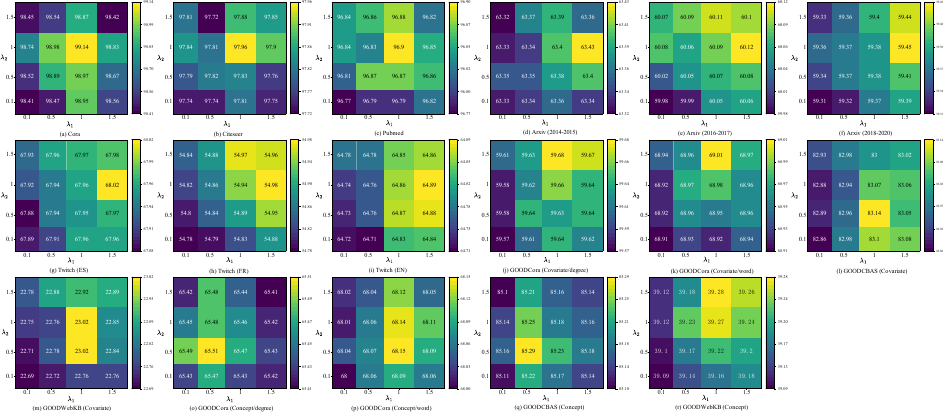}}
		\caption{The sensitivity of $\lambda_1$ and $\lambda_2$ to all datasets.}
		
		\label{figure11}
	\end{center}
\end{figure}
\subsection{Sensitivity to Hyperparameters}

CGRL involves three key hyperparameters: the branch number $K$, and the regularization weights $\lambda_1$ and $\lambda_2$, which control the strengths of $\mathcal{L}_{intra}$ and $\mathcal{L}_{inter}$, respectively. We next analyze the sensitivity of the proposed framework to these hyperparameters.

Figure \ref{figure10} presents the sensitivity of CGRL to the branch number $K$ on different datasets. For all datasets, we vary $K$ in $\{1,2,3,4,5\}$. Specifically, the best performance is consistently achieved around $K=3$ on Cora, Citeseer, and Pubmed, suggesting that a moderate number of branches is sufficient to capture diverse causal modulation patterns. When $K$ becomes larger, the gains tend to saturate or slightly decline, implying that excessive branches may introduce redundant modeling capacity and make optimization less effective. On GOODCora, the performance varies only mildly with $K$ and reaches its peak at $K=4$, further demonstrating the robustness of CGRL to the branch configuration under controlled OOD shifts.

We further study the sensitivity of the regularization weights $\lambda_1$ and $\lambda_2$ on Cora and Citeseer in Fig \ref{figure11}. In general, overly small values of $\lambda_1$ and $\lambda_2$ weaken the effects of intra-class aggregation and inter-class separation, making it difficult for the model to fully exploit the geometry-aware regularization imposed on the latent space. In contrast, excessively large values may over-emphasize these regularizers and impair the balance between structural disentanglement and supervised prediction. As shown in Figure \ref{figure11}, CGRL achieves the best overall performance at $\lambda_1=\lambda_2=1.0$ on Cora, and at $\lambda_1=1.5,\lambda_2=1.0$ on Arxiv. These results indicate that both $\mathcal{L}_{intra}$ and $\mathcal{L}_{inter}$ are important to the framework, while their relative strengths should be properly balanced to obtain the best generalization performance.


\subsection{Limitations}
A limitation of CGRL lies in the computational cost of its original pairwise objectives. Under naive full-graph evaluation, both $\mathcal{L}_{intra}$ and $\mathcal{L}_{inter}$ scale quadratically with the number of nodes, which poses substantial challenges on large graphs such as Pubmed and Arxiv. To make training practical in these settings, we employ a scalable prototype-and-sampling implementation. While this approximation significantly improves memory efficiency and runtime, it also introduces several limitations.

First, the prototype-based approximation in $\mathcal{L}_{intra}$ compresses rich intra-class pairwise relations into class-level summary statistics. Although this preserves the overall tendency toward class-wise compactness, it may fail to fully capture fine-grained local variations within a class, especially when class distributions are heterogeneous. As a result, the learned representation may not completely recover the ideal causal intra-class structure implied by the original full-pair objective. Second, the sampling-based approximation in $\mathcal{L}_{inter}$ only observes a subset of inter-class interactions at each optimization step. Consequently, some hard negative pairs or informative boundary cases may be under-sampled, which weakens the model’s ability to consistently enforce global inter-class separation. This may in turn leave residual spurious overlap in the latent space and partially limit the disentanglement between causal and non-causal components. Third, the stochasticity introduced by negative sampling can increase optimization variance across iterations. On large-scale graphs, this effect may lead to less stable mutual-information dynamics, as reflected by relatively larger Jitter values in some cases. In other words, although the scalable implementation preserves the overall optimization direction of the original objectives, it does not always preserve their full geometric fidelity or training stability.

These trade-offs indicate that the current scalable design, although effective in practice, does not yet fully resolve the tension between computational tractability and stable causal disentanglement. Developing more efficient optimization schemes that better preserve the structural faithfulness of the original pairwise objectives remains an important direction for future work.

\section{CONCLUSION}

In this paper, we study node-level out-of-distribution (OOD) generalization of graph neural networks from a causal perspective. We show that, under distribution shifts from environmental noise, GNNs suffer not only from degraded predictive performance but also from unstable representation learning, which we characterize as the \emph{Info-Jitter} phenomenon. To address these challenges, we formulate a node-classification-specific causal graph and use $do$-calculus from causal inference to block the interference of environmental noise, thereby deriving a deconfounded interventional objective for robust node-level prediction. Based on this formulation, we further develop a variational lower bound that provides a principled basis for disentangling intra-class causal structure from inter-class interference. Building on these insights, we propose CGRL, a causal-guided representation learning framework that integrates multi-branch causal modulation, intra-class aggregation, inter-class separation, energy-based reconstruction, and supervised prediction. Extensive experiments on diverse benchmark datasets demonstrate that CGRL consistently improves node-level OOD generalization and substantially alleviates the \emph{Info-Jitter} phenomenon. In the future, it would be valuable to design more efficient optimization schemes for large-scale graphs, strengthen the theoretical understanding of the relationship between causal disentanglement and mutual-information stability, and extend the proposed framework to broader graph learning settings.


\bibliographystyle{ACM-Reference-Format}
\bibliography{sample-base}

\appendix
\section{Proof of Theorem \ref{eq_backdoor}}\label{appendix}
\begin{proof}
	 For clarity, let $G$ denote the causal DAG in Figure \ref{image2a}, so as to avoid confusion with the input graph \(G\) and the ego-graph \(\mathcal{G}_v\). 	We prove the result using Rules 2 and 3 of $do$-calculus from causal inference \cite{primer}.
	\begin{itemize}
		\item Rule 2 (action--observation exchange) states that for disjoint variable sets
		$T, B, M, U$,
		\begin{equation}
			P(T\mid do(B),do(M),U)
			=
			P(T\mid do(B),M,U), if
			(T\perp\!\!\!\perp M\mid B,U)_{G_{\overline{B}\underline{M}}},
		\end{equation}
		where $G_{\overline{B}\underline{M}}$ denotes the graph obtained by deleting all incoming edges into $B$ and all outgoing edges from $M$.
		
		\item Rule 3 (insertion/deletion of actions) states that
		\begin{equation}
			P(T\mid do(B),do(M),U)
			=
			P(T\mid do(B),U), if (T\perp\!\!\!\perp M\mid B,U)_{G_{\overline{B},\overline{M(U)}}},
		\end{equation}
		where $U$ denotes the subset of $M$ that is not an ancestor of any variable in $U$ in $G_{\overline{B}}$.
	\end{itemize}
	
	We now apply these rules to the intervention on $\mathbf{Z}_p$. First, by the law of total probability for interventional distributions,
	\begin{equation}
		P(\mathbf{Y}\mid do(\mathbf{Z}_p))
		=
		\sum_{\mathcal{G}_v}
		P(\mathbf{Y}\mid do(\mathbf{Z}_p),\mathcal{G}_v)\,
		P(\mathcal{G}_v\mid do(\mathbf{Z}_p)).
		\label{eq:backdoor_step1}
	\end{equation}
	
	Second, $\mathcal{G}_v$ satisfies  $(\mathbf{Y}\perp\!\!\!\perp \mathbf{Z}_p \mid \mathcal{G}_v)_{G_{\underline{\mathbf{Z}_p}}}$, where $G_{\underline{\mathbf{Z}_p}}$ denotes the graph obtained by deleting all outgoing edges from $\mathbf{Z}_p$. Therefore, by Rule 2, we have
	\begin{equation}
		P(\mathbf{Y}\mid do(\mathbf{Z}_p),\mathcal{G}_v)
		=
		P(\mathbf{Y}\mid \mathbf{Z}_p,\mathcal{G}_v).
		\label{eq:backdoor_step2}
	\end{equation}
	
	Third, \(\mathcal{G}_v\) is a pre-treatment variable set and contains no descendants of \(\mathbf{Z}_p\). Hence intervening on \(\mathbf{Z}_p\) does not alter the generating mechanism of \(\mathcal{G}_v\). Equivalently, by Rule 3, we have
	\begin{equation}
		P(\mathcal{G}_v\mid do(\mathbf{Z}_p))
		=
		P(\mathcal{G}_v).
		\label{eq:backdoor_step3}
	\end{equation}
	
	Substituting Eqs.~\eqref{eq:backdoor_step2} and \eqref{eq:backdoor_step3} into Eq.~\eqref{eq:backdoor_step1}, we obtain
	\begin{equation}
		P(\mathbf{Y}\mid do(\mathbf{Z}_p))
		=
		\sum_{\mathcal{G}_v}
		P(\mathbf{Y}\mid \mathbf{Z}_p,\mathcal{G}_v)\,P(\mathcal{G}_v)
		=
		\mathbb{E}_{P(\mathcal{G}_v)}
		\big[P(\mathbf{Y}\mid \mathbf{Z}_p,\mathcal{G}_v)\big].
	\end{equation}
	This completes the proof.
\end{proof}

\section{Proof of Theorem \ref{theorem4.2}}\label{appendixA}
Before the proof of Eq. \ref{eq:3}, we first introduce several independence rules of SCM. Consider three random variables $T$, $B$, and $M$ in causal graph:

    (1) If there exists a causal chain structure $T\rightarrow B\rightarrow M$, then $T\perp\!\!\!\perp M|B$ (conditional independence),
    
    (2) If there exists a collider structure $T\rightarrow B\leftarrow M$, then $T\perp\!\!\!\perp M$ (unconditional independence),
    
    (3) If there exists a fork structure $T\leftarrow B\rightarrow M$, then $T\perp\!\!\!\perp M|B$ (conditional independence).
    
Therefore, by virtue of the log-likelihood, we have
\begin{equation}\label{eq:6}
    \begin{aligned}
    \log P(\mathbf{Y}|do(\mathbf{Z}_p)&=\log\sum_{\mathcal{G}_v} P(\mathbf{Y}|\mathbf{Z}_p, \mathcal{G}_v)P(\mathcal{G}_v)\\
    &=\log\sum_{\mathcal{G}_v, \mathbf{Z}_a, \mathbf{Z}_r}P(\mathbf{Y},\mathbf{Z}_a,\mathbf{Z}_r|\mathbf{Z}_p,\mathcal{G}_v)P(\mathcal{G}_v)\\
    &=\log \sum_{\mathcal{G}_v,\mathbf{Z}_a,\mathbf{Z}_r}P(\mathbf{Y}|\mathbf{Z}_p,\mathbf{Z}_a,\mathcal{G}_v)P(\mathbf{Z}_a|\mathbf{Z}_p)P(\mathbf{Z}_r|\mathbf{Z}_p)P(\mathcal{G}_v)\\
    &=\log \sum_{\mathcal{G}_v,\mathbf{Z}_a,\mathbf{Z}_r}P(\mathbf{Y}|\mathbf{Z}_p,\mathbf{Z}_a,\mathcal{G}_v)P(\mathbf{Z}_a|\mathbf{Z}_p)P(\mathbf{Z}_r|\mathbf{Z}_p)P(\mathcal{G}_v)\frac{\!Q_\phi(\mathbf{Z}_r,\mathbf{Z}_r,\mathcal{G}_v|\mathbf{Z}_p)}{Q_\phi(\mathbf{Z}_r,\mathbf{Z}_r,\mathcal{G}_v|\mathbf{Z}_p)}.
\end{aligned}
\end{equation}
We designate the Eq. \ref{eq:6} as $\mathcal{L}(\theta;\phi)$. Through calculation the posterior distribution $Q_\phi(\mathbf{Z}_a, \mathbf{Z}_r,\mathcal{G}_v|\mathbf{Z}_p)$, we obtain
\begin{align}
    Q_\phi(\mathbf{Z}_a, \mathbf{Z}_r,\mathcal{G}_v|\mathbf{Z}_p)=Q_\phi(\mathcal{G}|\mathbf{Z}_a, \mathbf{Z}_r, \mathbf{Z}_p) Q_\phi(\mathbf{Z}_a| \mathbf{Z}_r,\mathbf{Z}_p) Q_\phi(\mathbf{Z}_r|\mathbf{Z}_p).
\end{align}
Based on the aforementioned three rules, we can simplify this equation to
\begin{align}
     Q_\phi(\mathbf{Z}_a, \mathbf{Z}_r,\mathcal{G}_v|\mathbf{Z}_p)=Q_\phi(\mathcal{G}_v|\mathbf{Z}_p)Q_\phi(\mathbf{Z}_a|\mathbf{Z}_p)Q_\phi(\mathbf{Z}_r|\mathbf{Z}_p).
\end{align}
Therefore, from Jensen's inequality, we have 
\begin{equation}\label{eq:9}
    \begin{aligned}
        \mathcal{L}(\theta;\phi)&\geq \sum_{\mathcal{G}_v,\mathbf{Z}_a,\mathbf{Z}_r} Q_\phi(\mathbf{Z}_a, \mathbf{Z}_r,\mathcal{G}_v|\mathbf{Z}_p)\log P(\mathbf{Y}|\mathbf{Z}_p,\mathbf{Z}_a,\mathcal{G}_v)P(\mathbf{Z}_a|\mathbf{Z}_p)P(\mathbf{Z}_r|\mathbf{Z}_p)P(\mathcal{G}_v)\frac{1}{ Q_\phi(\mathbf{Z}_a, \mathbf{Z}_r,\mathcal{G}_v|\mathbf{Z}_p)}\\
        &=\sum_{\mathcal{G}_v,\mathbf{Z}_a,\mathbf{Z}_r}Q_\phi(\mathcal{G}_v|\mathbf{Z}_p)Q_\phi(\mathbf{Z}_a|\mathbf{Z}_p)Q_\phi(\mathbf{Z}_r|\mathbf{Z}_p)\log P(\mathbf{Y}|\mathbf{Z}_p,\mathbf{Z}_a,\mathcal{G}_v)\\
        &+\sum_{\mathcal{G}_v,\mathbf{Z}_a,\mathbf{Z}_r}Q_\phi(\mathcal{G}_v|\mathbf{Z}_p)Q_\phi(\mathbf{Z}_a|\mathbf{Z}_p)Q_\phi(\mathbf{Z}_r|\mathbf{Z}_p)\log P(\mathbf{Z}_a|\mathbf{Z}_p)\\
        &+\sum_{\mathcal{G}_v,\mathbf{Z}_a,\mathbf{Z}_r}Q_\phi(\mathcal{G}_v|\mathbf{Z}_p)Q_\phi(\mathbf{Z}_a|\mathbf{Z}_p)Q_\phi(\mathbf{Z}_r|\mathbf{Z}_p)\log P(\mathbf{Z}_r|\mathbf{Z}_p)\\
        &+\sum_{\mathcal{G}_v,\mathbf{Z}_a,\mathbf{Z}_r}Q_\phi(\mathcal{G}_v|\mathbf{Z}_p)Q_\phi(\mathbf{Z}_a|\mathbf{Z}_p)Q_\phi(\mathbf{Z}_r|\mathbf{Z}_p)\log P(\mathcal{G}_v)\\
       & -\sum_{\mathcal{G}_v,\mathbf{Z}_a,\mathbf{Z}_r}Q_\phi(\mathcal{G}_v|\mathbf{Z}_p)Q_\phi(\mathbf{Z}_a|\mathbf{Z}_p)Q_\phi(\mathbf{Z}_r|\mathbf{Z}_p) \log Q_\phi(\mathcal{G}_v|\mathbf{Z}_p)Q_\phi(\mathbf{Z}_a|\mathbf{Z}_p)Q_\phi(\mathbf{Z}_r|\mathbf{Z}_p).
    \end{aligned}
\end{equation}
Since the marginal distribution probability $\sum_{\mathcal{G}_v}Q_\phi(\mathcal{G}_v|\mathbf{Z}_p)=1$, $\sum_{\mathbf{Z}_a}Q_\phi(\mathbf{Z}_a|\mathbf{Z}_p)=1$ and $\sum_{\mathbf{Z}_r}Q_\phi(\mathbf{Z}_r\allowbreak |\mathbf{Z}_p)=1$ the final step of Eq. \ref{eq:9} can be simplified to
\begin{equation}
    \begin{aligned}
        \mathcal{L}(\theta,\phi)&\geq\sum_{\mathcal{G}_v,\mathbf{Z}_a}Q_\phi(\mathcal{G}_v|\mathbf{Z}_p)Q_\phi(\mathbf{Z}_a|\mathbf{Z}_p)\log P(\mathbf{Y}|\mathbf{Z}_p,\mathbf{Z}_a,\mathcal{G}_v)
        +\sum_{\mathbf{Z}_a}Q_\phi(\mathbf{Z}_a|\mathbf{Z}_p)\log P(\mathbf{Z}_a|\mathbf{Z}_p)\\
       & +\sum_{\mathbf{Z}_r}Q_\phi(\mathbf{Z}_r|\mathbf{Z}_p)\log P(\mathbf{Z}_r|\mathbf{Z}_p)\
        +\sum_{\mathcal{G}_v}Q_\phi(\mathcal{G}_v|\mathbf{Z}_p)\log P(\mathcal{G}_v)
        -\sum_{\mathcal{G}_v}Q_\phi(\mathcal{G}_v|\mathbf{Z}_p) \log Q_\phi(\mathcal{G}_v|\mathbf{Z}_p)\\
        &-\sum_{\mathbf{Z}_a}Q_\phi(\mathbf{Z}_a|\mathbf{Z}_p)\log Q_\phi(\mathbf{Z}_a|\mathbf{Z}_p)
        -\sum_{\mathbf{Z}_r}Q_\phi(\mathbf{Z}_r|\mathbf{Z}_p) \log Q_\phi(\mathbf{Z}_r|\mathbf{Z}_p)\\
        &=\mathbb{E}_{Q_\phi(\mathcal{G}_v|\mathbf{Z}_p)}\mathbb{E}_{Q_\phi(\mathbf{Z}_a|\mathbf{Z}_p)}\log P(\mathbf{Y}|\mathbf{Z}_p, \mathbf{Z}_a,\mathcal{G}_v)-\text{KL}(Q_\phi(\mathcal{G}_v|\mathbf{Z}_p)||P(\mathcal{G}_v))\\&-\text{KL}(Q_\phi(\mathbf{Z}_a|\mathbf{Z}_p)||P(\mathbf{Z}_a|\mathbf{Z}_p))-\text{KL}(Q_\phi(\mathbf{Z}_r|\mathbf{Z}_p)||P(\mathbf{Z}_r|\mathbf{Z}_p)).
    \end{aligned}
\end{equation}
In the end, through these derivations, we prove Theorem \ref{theorem4.2} and obtain the theoretical lower bound to be optimized.

\section{Proof of Proposition \ref{prop_intra}}\label{appendixB}

\begin{proof}
	Let $d$ denote the dimensionality of the latent representation space. Although the original graph data resides in a non-Euclidean topological domain, the GNN encoder maps the nodes into a Euclidean latent space $\mathbb{R}^d$. For a specific class $c \in \mathcal{Y}$, we characterize the class-conditional variational posterior $Q_\phi(\mathbf{Z}_a|\mathbf{Z}_p)$ by its empirical mean $\boldsymbol{\mu}_Q^{(c)}$ and covariance matrix $\boldsymbol{\Sigma}_Q^{(c)}$. Based on the analytical modeling choice in Assumption \ref{assumption_causal}, the corresponding prior is the isotropic Gaussian $P(\mathbf{Z}_a|c) = \mathcal{N}(\boldsymbol{\mu}_c, \sigma_p^2 \mathbf{I})$.
	
	The Kullback-Leibler (KL) divergence between two $d$-dimensional multivariate Gaussians, $Q = \mathcal{N}(\boldsymbol{\mu}_0, \boldsymbol{\Sigma}_0)$ and $P = \mathcal{N}(\boldsymbol{\mu}_1, \boldsymbol{\Sigma}_1)$, is formulated as
	\begin{equation}\label{eq:general_kl}
		\text{KL}(Q \parallel P) = \frac{1}{2} \left[ \text{Tr}(\boldsymbol{\Sigma}_1^{-1} \boldsymbol{\Sigma}_0) + (\boldsymbol{\mu}_0 - \boldsymbol{\mu}_1)^\top \boldsymbol{\Sigma}_1^{-1} (\boldsymbol{\mu}_0 - \boldsymbol{\mu}_1) - d + \log \frac{|\boldsymbol{\Sigma}_1|}{|\boldsymbol{\Sigma}_0|} \right].
	\end{equation}
	
	Substituting the specified prior and posterior parameters for class $c$ into Eq. \ref{eq:general_kl}, the inverse covariance matrix of the prior is $\boldsymbol{\Sigma}_1^{-1} = \frac{1}{\sigma_p^2} \mathbf{I}$. This isotropic structure reduces the Mahalanobis distance term to a scaled squared Euclidean distance. With the differential entropy defined as $H(Q_\phi^{(c)}) = \frac{1}{2} \log(|\boldsymbol{\Sigma}_Q^{(c)}|) + \frac{d}{2} \log(2\pi e)$, the log-determinant ratio expands accordingly. Absorbing the terms independent of the optimization parameters $\phi$ into a const, the KL divergence for class $c$ reduces to
	\begin{equation} \label{eq:kl_gaussian_reduced}
		\text{KL}(Q^{(c)}_\phi(\mathbf{Z}_a|\mathbf{Z}_p) \parallel P^{(c)}_\theta(\mathbf{Z}_a|\mathbf{Z}_p)) = \frac{1}{2\sigma_p^2} \Big( \text{Tr}(\boldsymbol{\Sigma}_Q^{(c)}) + \|\boldsymbol{\mu}_Q^{(c)} - \boldsymbol{\mu}_c\|_2^2 \Big) - H(Q_\phi^{(c)}) + const.
	\end{equation}
	
	The variance penalty in Eq. \ref{eq:kl_gaussian_reduced} is governed by the covariance trace $\text{Tr}(\boldsymbol{\Sigma}_Q^{(c)})$. To connect this trace with the empirical objective, we evaluate the expected pairwise squared Euclidean distance between two conditionally independent nodes $i, j \in \mathcal{V}_c$, whose representations are drawn from the class-conditional posterior: $\mathbf{z}_{a,i}, \mathbf{z}_{a,j} \sim Q_\phi(\mathbf{Z}_a|\mathbf{Z}_p)$. Expanding the quadratic form yields
	\begin{align} \label{eq:trace_equivalence}
		\mathbb{E}_{\mathbf{z}_{a,i}, \mathbf{z}_{a,j} \sim Q_\phi} \Big[ \|\mathbf{z}_{a,i} - \mathbf{z}_{a,j}\|_2^2 \Big] &= \mathbb{E}_{i, j} \Big[ \|(\mathbf{z}_{a,i} - \boldsymbol{\mu}_Q^{(c)}) - (\mathbf{z}_{a,j} - \boldsymbol{\mu}_Q^{(c)})\|_2^2 \Big] \nonumber \\
		&= \mathbb{E}_{i}[\|\mathbf{z}_{a,i} - \boldsymbol{\mu}_Q^{(c)}\|_2^2] + \mathbb{E}_{j}[\|\mathbf{z}_{a,j} - \boldsymbol{\mu}_Q^{(c)}\|_2^2] - 2\mathbb{E}_{i, j}[\langle \mathbf{z}_{a,i} - \boldsymbol{\mu}_Q^{(c)}, \mathbf{z}_{a,j} - \boldsymbol{\mu}_Q^{(c)} \rangle].
	\end{align}
	Since nodes $i$ and $j$ are sampled conditionally independently, their cross-covariance expectation vanishes ($\mathbb{E}[\mathbf{z}_{a,i} - \boldsymbol{\mu}_Q^{(c)}] = \mathbf{0}$). Based on the definition of the covariance matrix, the expected squared distance to the intra-class mean is $\mathbb{E}[\|\mathbf{z}_a - \boldsymbol{\mu}_Q^{(c)}\|_2^2] = \text{Tr}(\boldsymbol{\Sigma}_Q^{(c)})$. Eq. \ref{eq:trace_equivalence} thus simplifies to
	\begin{equation} \label{eq:trace_empirical}
		\mathbb{E}_{\mathbf{z}_{a,i}, \mathbf{z}_{a,j} \sim Q_\phi} \Big[ \|\mathbf{z}_{a, i} - \mathbf{z}_{a, j}\|_2^2 \Big] = 2\text{Tr}(\boldsymbol{\Sigma}_Q^{(c)}).
	\end{equation}
	By definition, the empirical intra-class alignment loss computes this expected pairwise distance averaged over all classes $c \in \mathcal{Y}$. Substituting Eq. \ref{eq:trace_empirical} into the empirical formulation yields the exact identity $\mathcal{L}_{intra} = 2\mathbb{E}_{c \in \mathcal{Y}}[\text{Tr}(\boldsymbol{\Sigma}_Q^{(c)})]$.
	
	Taking the expectation of the reduced KL objective (Eq. \ref{eq:kl_gaussian_reduced}) over all classes $c \in \mathcal{Y}$ and substituting the trace identity, we obtain
	\begin{equation}
		\mathbb{E}_{c \in \mathcal{Y}} \Big[ \text{KL}(Q_\phi^{(c)} \parallel P^{(c)}) \Big] = \frac{1}{4\sigma_p^2} \mathcal{L}_{intra} + \mathbb{E}_{c \in \mathcal{Y}} \left[ \frac{1}{2\sigma_p^2}\|\boldsymbol{\mu}_Q^{(c)} - \boldsymbol{\mu}_c\|_2^2 - H(Q_\phi^{(c)}) \right] + const.
	\end{equation}
	
	Under Assumption \ref{assumption_causal}, the differential entropy is lower-bounded ($H(Q_\phi^{(c)}) \ge \xi$), and the squared distance to the prototype $\|\boldsymbol{\mu}_Q^{(c)} - \boldsymbol{\mu}_c\|_2^2$ is inherently non-negative. While $\mathcal{L}_{intra}$ alone does not strictly upper-bound the entire KL divergence due to the mean discrepancy term, minimizing $\mathcal{L}_{intra}$ explicitly suppresses the covariance trace. Therefore, $\mathcal{L}_{intra}$ serves as a principled surrogate that directly regulates the intra-class variance penalty within the KL divergence, effectively encouraging causal invariance.
\end{proof}

\end{document}